\documentclass[a4paper]{article}
\usepackage[preprint,nonatbib]{neurips}

\usepackage{amsmath}
\usepackage{amsthm}
\usepackage{amssymb}
\usepackage{hyperref}
\usepackage{cleveref}
\usepackage{dsfont} 
\usepackage{enumitem}

\numberwithin{equation}{section}

\newtheorem{theorem}{Theorem}[section]
\newtheorem{definition}[theorem]{Definition}

\newtheorem{lemma}[theorem]{Lemma}

\theoremstyle{remark}
\newtheorem{remark}[theorem]{Remark}
\newtheorem*{rem*}{Remark}

\newcommand{\Alg}{\operatorname{Alg}}
\newcommand{\deterministic}{\mathrm{det}}
\newcommand{\MonteCarlo}{\mathrm{MC}}

\newcommand{\avsum}{\mathop{\mathpalette\avsuminner\relax}\displaylimits}

\makeatletter
\newcommand\avsuminner[2]{%
	{\sbox0{$\m@th#1\sum$}%
		\vphantom{\usebox0}%
		\ooalign{%
			\hidewidth
			\smash{\vrule height\dimexpr\ht0+1pt\relax depth\dimexpr\dp0+1pt\relax}%
			\hidewidth\cr
			$\m@th#1\sum$\cr
		}%
	}%
}
\makeatother

\newcommand{\nrow}[1]{%
	\relax
	\vcenter to 0pt{%
		\vss
		\kern-1.5ex
		\rlap{$\left.\vphantom{\begin{matrix}0\\\vdots\\0\end{matrix}}\kern3em\right\}#1$}%
		\vss
	}%
}

\makeatletter
\newcommand{\subjclass}[2][1991]{%
	\let\@oldtitle\@title%
	\gdef\@title{\@oldtitle\footnotetext{#1 \emph{Mathematics subject classification.} #2}}%
}

\newcommand{\keywords}[1]{%
	\let\@@oldtitle\@title%
	\gdef\@title{\@@oldtitle\footnotetext{\emph{Key words and phrases.} #1.}}%
}
\makeatother

\title{Sampling Complexity of Deep Approximation Spaces}

\date{}

\author{Ahmed Abdeljawad\footnotemark[1]
	\; and\; Philipp Grohs%
	\thanks{Johann Radon Institute for Computational and Applied Mathematics (RICAM),
    Austrian Academy of Sciences,
    Altenbergerstr. 69, 4040 Linz, Austria}\;
	\thanks{Faculty of Mathematics,
		University of Vienna,
		Oskar-Morgenstern-Platz~1,
		A-1090 Vienna, Austria}\;
	\thanks{Research Platform Data Science @ Uni Vienna,
		Währinger Straße 29/S6,
		A-1090 Vienna, Austria}
	\\[1ex]
	ahmed.abdeljawad@oeaw.ac.at, \;
	philipp.grohs@univie.ac.at
}

\begin{document}
\maketitle	
\begin{abstract}
While it is well-known that neural networks enjoy excellent approximation capabilities, it remains a big challenge to compute such approximations from point samples. Based on tools from Information-based complexity, recent work by Grohs and Voigtlaender [Journal of the FoCM (2023)] developed a rigorous framework for assessing this so-called "theory-to-practice gap". More precisely, in that work it is shown that there exist functions that can be approximated by neural networks with ReLU activation function at an arbitrary rate while requiring an exponentially growing (in the input dimension) number of samples for their numerical computation. The present study extends these findings by showing analogous results for the ReQU activation function. 
\end{abstract}
	
\section{Introduction}\label{sec:Introduction}

Deep learning has made remarkable impacts in a wide range of tasks such
as speech recognition, computer vision, natural language processing
and many others. 
A theoretical understanding
of the reason behind the empirical success of
deep neural networks is still missing. 
This includes the question of \emph{expressivity},
which
has been studied and used since 1989's,
dating back to the foundational works of
Cybenko \cite{Cybenko89}, Hornik, Stinchcombe and White
\cite{Hornic89} and Hornic \cite{Hornic}.

Recently, several researchers are focusing on the theoretical
foundation of neural networks from a mathematical point of view.
Indeed, many papers treat the role of depth, width,
activation functions
and architecture in approximating a given function
e.g.,
\cite{abdeljawad22, abdeljawad23, abdeljawad21, beck2018solving, Berg2018,
	Berner2020, Chen2019,elbrachter2018dnn,
	Gonon2019, Grohs2020, Grohs2019, GrohsPEB2019,
	Hana2018, Kutyniok2019, Magill2018}.
More recent works focus predominantly on the number of samples necessary
to reach certain accuracy while approximating a given function.
Some examples ---without any claim of completeness--- are
\cite{Berner2022}, in which the minimal number
of training samples when
training an algorithm
(to guarantee a given uniform accuracy 
when approximating functions with  ReLU neural networks)	
scales exponentially both in the depth and the
input dimension of the network architecture.
Or \cite{grohs2021} which considers
the sampling complexity problem 
when approximating or integrating functions that can be well approximated
by ReLU neural networks, where
authors used different norms when they measured accuracy.
Following this line of work, we focus on the impact of
the chosen activation function on the problem
of sample complexity where the accuracy is measured 
with respect to the uniform norm.

Determining the sample complexity for an algorithm $\mathcal{A}$
in order to achieve an optimal accuracy of $\epsilon$ while approximating
a given function $u$ with deep neural networks depends on several factors.
For instance, the space where the function $u$ belongs,
the norm where the accuracy is measured and the characteristics of
the algorithm $\mathcal{A}$.
Here the algorithm is an element of so-called
\emph{neural network approximation spaces}
$A^{\alpha,p}_{\boldsymbol{\ell},\boldsymbol{c},\varrho_2}([0,1]^d)$
which classifies functions depending
on the decay  of the approximation error by neural networks
with respect to the number of non-zero weights.

Basically, a function $u$ belongs to the unit ball
of $A^{\alpha,p}_{\boldsymbol{\ell},\boldsymbol{c},\varrho_2}([0,1]^d)$,
denoted by $U^{\alpha,p}_{\boldsymbol{\ell},\boldsymbol{c},\varrho_2}([0,1]^d)$,
if for any $n \in\mathbb{N}$
there exists a ReQU neural network
with at most $n$ nonzero weights of
magnitude at most $\boldsymbol{c}(n)$ and at most $\boldsymbol{\ell}(n)$ layers
approximating $u$ with accuracy less than $n^{-\alpha}$
in the $L^p ([0, 1]^d )$ norm,	
see Section \ref{sub:NeuralNetworkApproxSpaces} for more details.
Roughly speaking, a function $u$ belongs to 
$A^{\alpha,p}_{\boldsymbol{\ell},\boldsymbol{c},\varrho_2}([0,1]^d)$
(for large $\alpha$) means that 
$u$ can be well approximated by ReQU neural networks.
Neural networks approximation spaces were
first introduced in \cite{NNApproximationSpaces}
and extended
by Grohs and Voightlaender in \cite{grohs2021} to more general settings
taking into account the magnitude of the network weights and 
the maximal depth of the network.

In practice, deep neural networks use more parameters than samples
while training,	and	therefore have the capacity to overfit to the training set.
Usually one needs regularization during the training process,
or theoretically, more information about the characteristics of function
that the network is approximating.
The latter can be encoded in the \emph{target class} $U\subset C([0,1]^d)$
which contains the ground truth $u$.
In the current paper, we are interested in the complexity
of approximating a given solution mapping $S : U \rightarrow Y$,
such that $U\subset C([0,1]^d)$ and $Y$ is a Banach space,
that can be achieved by any algorithm
using only $m$ samples.
Formally, an algorithm using $m$ samples can be described by a set of samples
given by $X=(X_1, \dots, X_m)$
such that $X_i\in [0,1]^d$ for any $i\in \{1, \dots, m\}$ and
a map  $Q :\mathbb{R}^m \rightarrow Y$ such that 
\[
A(u) = Q(u(X_1), \dots, u(X_m)), \text{ for any } u\in U.
\]
The set of all such algorithms is denoted by $Alg_m (U, Y )$.
Furthermore, the optimal order for  approximating the mapping
$S : U \rightarrow Y$ using point samples can be determined as follows:

$$
err(U, S):=\sup \left\{\beta \geq 0: \exists C>0\;\forall m \in \mathbb{N}:
\inf _{A \in \operatorname{Alg}_{m}(U, Y)} \sup _{u \in U}
\|A(u)-S(u)\|_{Y} \leq C \cdot m^{-\beta}\right\}.
$$	
More details about the classes of
algorithms that we consider in this paper can be found
in Section \ref{sub:DeterministicAlgorithms} below.
The Banach space $Y$ gives information about where one would like to measure 
the error, in our analysis we select the uniform norm,
that is, $Y:= L^\infty([0,1]^d)$.
The study of the generalization error, that is,  $Y:= L^2([0,1]^d)$
is an interesting question for future work.
The choice of the activation function in the network is very critical.
Compared to the main theorems in
\cite[Section 4, Section 5]{grohs2021},
our choice of ReQU as an  activation function has a similar
impact on the sampling complexity.
We see these observations in the following
simplified version of our main results in
\Cref{thm:ErrorBoundUniformApproximation}
and \Cref{thm:UniformApproximationHardness}.
\begin{theorem}
	Let $\boldsymbol{c}: \mathbb{N} \rightarrow \mathbb{N}\cup\{\infty\}$
	be of the form
	$\boldsymbol{c}(n)\asymp n^\theta\cdot \left(\log(2n)\right)^\kappa$
	for certain $\theta \geq 0$ and $\kappa\in\mathbb{R}$,
	and let $\boldsymbol{\ell}: \mathbb{N} \rightarrow \mathbb{N}_{\geq5}\cup\{\infty\}$
	such that  $\boldsymbol{\ell}^\ast:=\sup_{n \in \mathbb{N}}\boldsymbol{\ell}(n)<\infty$.
	Then, we get
\[
\frac{1}{d} \cdot \frac{\alpha}{\alpha + (2^{{{\boldsymbol{\ell}}^\ast}}  -1)(\theta + 1/2)}\leq
err(U_{\boldsymbol{\ell},
\boldsymbol{c}, \varrho_2}^{\alpha,\infty}([0,1]^d),\iota_\infty)
\leq 
\frac{64}{d} \cdot \frac{\alpha}{8\alpha + (2^{{{\boldsymbol{\ell}}^\ast}}  -1)(\theta + 1/2)}.
\]
\end{theorem}
We expect that (a qualitatively similar version of) our result holds for 
neural networks activated by Rectified Power Unit (RePU)
$x\mapsto \max(0, x)^p$, for any $p\in \mathbb{N}_{\geq 2}$,
since, for any integer $p\geq 2$
$\max(0, x)^p$ can be exactly represented with a
certain ReQU neural network.
The sampling complexity of learning ReLU neural networks 
has recently been studied in 
\cite{Golowich2018,NNApproximationSpaces, grohs2021,Zhang2019}.
In our paper we consider the ReQU activation function,
thereby going a first step towards considering general activation functions.
\subsection{Notation}%
\label{sub:Notation}

For $n \in \mathbb{N}$, we write $\underline{n} := \{ 1,2,\dots,n \}$.
For any finite set $I \neq \varnothing  $ and any sequence
$(a_i)_{i \in I} \subset \mathbb{R}$,
we define $\avsum_{i \in I} a_i := \frac{1}{|I|} \sum_{i \in I} a_i$.
The expectation of a random variable $X$ will be denoted
by $\mathbb{E}[X]$.
For a subset $M \subset \mathbb{R}^d$,
we write $\overline{M}$ for the closure of $M$
and $M^\circ$ for the interior of $M$.

\subsection{Organization of the Paper}

The organization of this paper is as follows.
In Section \ref{sec:ApproximationSpacesAndSamplingComplexity},
we give a mathematical definition of the architecture and the realization of
ReQU neural networks. Furthermore, we define
notions of neural networks approximation space with ReQU activation function.
We recall the deterministic and Monte Carlo optimal orders.
In Section \ref{sec:HatFunctionConstruction}
we show that the unit ball of the ReQU neural networks approximation space
contains a large family of ``hat functions" constructed by the use of 
ReQU activation function.
Our main results can be found in Sections \ref{sec:UniformApproximationHardness}
and \ref{sec:UniformApproximationErrorBounds}
where we develop error bounds and hardness results for uniform approximation,
in view of ReQU activation function.

\section{Sampling complexity on neural network approximation spaces}%
\label{sec:ApproximationSpacesAndSamplingComplexity}

In this section,
we will introduce the abstract setup and necessary notation
that we will consider throughout
this paper. First of all, we will define neural networks
from a mathematical point of view.
We distinguish between a neural network as a set of weights
and the realization of the neural network as the associated function.
Furthermore, we formally introduce
the neural network approximation spaces
$A^{\alpha,p}_{\boldsymbol{\ell},\boldsymbol{c},\varrho_2}$.
Then we briefly recall the framework of information based complexity,
such as the deterministic and randomized (Monte Carlo) algorithms
and the optimal order of convergence with respect to point samples.

\subsection{Mathematical definition of neural networks}%
\label{sub:NeuralNetworks}
From a functional analysis point of view, \emph{neural networks} are
functions of compositional form that result from repeatedly applying
affine maps and a non linear map called activation function.
In the definition of  neural network,
we differentiate between the \emph{architecture} which is a tuple of matrices and
vectors, describing the parameters of the neural network,
and the \emph{realization} which is the associated functions.
Concretely, we make the following definition:

\begin{definition}
	Let $L\in \mathbb{N}$ and $ N_0,\dots, N_L \in \mathbb{N}$.
	A neural network $\Phi$ with input dimension $N_0$ and $L$ layers
	is a sequence of matrix-vector tuples
	$$
	{\Phi = \big( (A_1,b_1), \dots, (A_L,b_L) \big)},
	$$
	where $A_\ell \in \mathbb{R}^{N_\ell \times N_{\ell-1}}$ and $b_\ell \in \mathbb{R}^{N_\ell}$
	for any $\ell \in \{1, \dots, L	\}$.
	
	We refer to $(N_0,\dots,N_L) \in \mathbb{N}^{L+1}$
	as the \emph{architecture} of $\Phi$.
	We call $L(\Phi) := L$
	the \emph{number of layers}
	\footnote{Note that the number of \emph{hidden} layers is given by $H = L-1$.}
	of $\Phi$,
	and ${W(\Phi) := \sum_{j=1}^L (\| A_j \|_{\ell^0} + \| b_j \|_{\ell^0})}$
	denotes the \emph{number of (non-zero) weights} of $\Phi$.
	We denote by $\| A \|_{\ell^0}$ the number of non-zero entries
	of a matrix (or vector) $A$.
	We write $d_{\mathrm{in}}(\Phi) := N_0$
	and $d_{\mathrm{out}}(\Phi) := N_L$
	for the \emph{input and output dimension} of $\Phi$, respectively.
	Let
	$\| \Phi \|_{\mathcal{NN}} := \max_{j = 1,\dots,L}
	\max \{ \| A_j \|_{\infty}, \| b_j \|_{\infty} \}$,
	where ${\| A \|_{\infty} := \max_{i,j} |A_{i,j}|}$.
	
	The \emph{realization} of $\Phi$
	is the map $R_{\varrho} \Phi$, where $\varrho$
	is the so-called
	\emph{activation function}.
	In this paper, we will consider the
	\emph{rectified Quadratic unit (ReQU)}
	${\varrho_2 : \mathbb{R} \to \mathbb{R}, x \mapsto \max(0, x )^2}$.
	The function $R_{\varrho_2} \Phi : \mathbb{R}^{N_0} \to \mathbb{R}^{N_L}$
	computed by the network $\Phi$ is given by
	\[
	R_{\varrho_2} \Phi
	:= T_L \circ (\varrho_2 \circ T_{L-1}) \circ \cdots \circ (\varrho_2 \circ T_1)
	\quad \text{where} \quad
	T_\ell \, x = A_\ell \, x + b_\ell .
	\]
\end{definition}
	
\subsection{Neural network approximation spaces}%
\label{sub:NeuralNetworkApproxSpaces}
The notion of \emph{neural network approximation spaces}
was originally introduced in
\cite{NNApproximationSpaces}, where $\Sigma_n$ was taken
to be a family of neural networks
of increasing complexity.
We use more restrictive class compared to the class
in \cite{NNApproximationSpaces}.
Indeed, we impose some growth conditions on the size
of the network weights.
In our paper we use a similar class of 
neural network approximation spaces
to the one given in \cite{grohs2021}.
Namely, given a fixed input dimension $d \in \mathbb{N}$ and
non-decreasing functions
${\boldsymbol{\ell} : \mathbb{N} \to \mathbb{N}_{\geq 3} \cup \{ \infty \}}$
and
${\boldsymbol{c}} : \mathbb{N} \to \mathbb{N} \cup \{ \infty \}$
called the \textit{depth-growth function}
and the \textit{coefficient growth function}, respectively.
We define $\Sigma^{\boldsymbol{\ell}, \boldsymbol{c},\varrho_2}_n$
as follows:
\[
\Sigma^{\boldsymbol{\ell}, \boldsymbol{c}, \varrho_2}_n
:= \Big\{
R_{\varrho_2} \Phi
\,\, \colon
\begin{array}{l}
	\Phi \text{ NN with }
	d_{\mathrm{in}}(\Phi) = d,
	d_{\mathrm{out}}(\Phi) = 1, \\
	W(\Phi) \leq n,
	L(\Phi) \leq {\boldsymbol{\ell}}(n),
	\| \Phi \|_{\mathcal{NN}} \leq {\boldsymbol{c}}(n)
\end{array}
\Big\} .
\]
	
In order to characterize the best possible convergence rate
with respect to the number of samples
we need to define the following quantities:
\begin{align}
	{\boldsymbol{\ell}}^\ast
	&:= \sup_{n\in\mathbb{N}} {\boldsymbol{\ell}}(n)
	\in \mathbb{N} \cup \{ \infty \} ;
	\label{eq:MaximalDepth}
	\\
	{\boldsymbol{c}}^\ast
	&:= \sup_{n\in\mathbb{N}} {\boldsymbol{c}}(n)
	\in \mathbb{N} \cup \{ \infty \} .
	\label{eq:MaximalCoef}
\end{align}
Furthermore, let
\begin{equation}
	\begin{aligned}
		\gamma^{\flat} ({\boldsymbol{\ell}},{\boldsymbol{c}})
		& := \sup
		\Big\{
		\gamma \geq 0
		\colon
		\exists \, L \in \mathbb{N}_{\leq {\boldsymbol{\ell}}^\ast},\, C > 0,
		 \forall \, n \in \mathbb{N}
		:&
		n^{\gamma}
		\leq
		C\,  \boldsymbol{c}(n)^{2^{{L}}-1}\cdot	
		n^{(2^{{L}}  -1)/2}
		\Big\} ,
		\\[1ex]
		\gamma^{\sharp} ({\boldsymbol{\ell}}, {\boldsymbol{c}})
		&:= \inf
		\Big\{
		\gamma \geq 0
		\colon
		\exists \; C > 0
		\; \forall \, n \in \mathbb{N} ,\, L \in \mathbb{N}_{\leq {\boldsymbol{\ell}}^\ast} :&
		\boldsymbol{c}(n)^{2^{{L}}-1}\cdot	
		n^{(2^{{L}}  -1)/2}
		\leq C\,  n^{\gamma}
		\Big\} .   
	\end{aligned}
	\label{eq:GammaDefinition}
\end{equation}
	
\begin{remark}\label{rem:GammaRemark}
	In view of the construction of
	$\gamma^{\flat}({\boldsymbol{\ell}},{\boldsymbol{c}})$
	and
	$\gamma^{\sharp}({\boldsymbol{\ell}},{\boldsymbol{c}})$
	it follows that
	$$
	\gamma^{\flat}
	({\boldsymbol{\ell}},{\boldsymbol{c}})
	\leq \gamma^{\sharp}
	({\boldsymbol{\ell}},{\boldsymbol{c}}).
	$$
	Furthermore, since we will only consider settings in which
	${\boldsymbol{\ell}}^\ast \geq 5$,
	we always have
	$$
	{\boldsymbol{\ell}}^\ast
	\leq
	\gamma^{\flat}({\boldsymbol{\ell}},
	{\boldsymbol{c}})
	\leq
	\gamma^{\sharp}({\boldsymbol{\ell}},
	{\boldsymbol{c}}) .
	$$
	Obviously if ${\boldsymbol{\ell}}$ is unbounded,
	that is ${\boldsymbol{\ell}}^\ast = \infty$, then
	$\gamma^{\flat}({\boldsymbol{\ell}},{\boldsymbol{c}})
	= \gamma^{\sharp}({\boldsymbol{\ell}},{\boldsymbol{c}})
	= \infty$.
	We remark that in the case where ${\boldsymbol{\ell}}^\ast< \infty$
	and
	$\boldsymbol{c}(n)\asymp n^\theta\cdot \left(\log(2n)\right)^\kappa$
	for certain $\theta \geq 0$ and $\kappa\in\mathbb{R}$, then
	\[
	\gamma^{\flat}
	({\boldsymbol{\ell}},{\boldsymbol{c}})
	= \gamma^{\sharp}
	({\boldsymbol{\ell}},{\boldsymbol{c}})
	=
	(2^{{{\boldsymbol{\ell}}^\ast}}  -1)(\theta + 1/2).
	\]
	Indeed, for a fixed $C\geq 1$ 
	and any $L\in \mathbb{N}_{\leq {\boldsymbol{\ell}}^\ast}$,
	we have
	\begin{align*}
	 \frac{C\,  n^{\gamma}}{
	 	\boldsymbol{c}(n)^{2^{{L}}-1}\cdot	
	 	n^{(2^{{L}}  -1)/2}} 
 	\geq 
 	\frac{C\,  n^\gamma}{
 		\boldsymbol{c}(n)^{2^{\boldsymbol{\ell}^\ast}-1}\cdot
 		n^{(2^{\boldsymbol{\ell}^\ast}  -1)/2}}  
 	\end{align*}
  	and that 
 	\begin{align*}
 	\frac{C\,  n^\gamma}{
 		\boldsymbol{c}(n)^{2^{\boldsymbol{\ell}^\ast}-1}\cdot
 		n^{(2^{\boldsymbol{\ell}^\ast}  -1)/2}}  
 	\asymp 
 	\frac{C\,  n^{\gamma}}{
 		\log(2n)^{\kappa(2^{{{\boldsymbol{\ell}}^\ast}}-1)}\cdot	
 		n^{(2^{{{\boldsymbol{\ell}}^\ast}}  -1)(\theta + 1/2)}} 
 	\xrightarrow[n]{}
 	C\,  n^{\gamma - (2^{{{\boldsymbol{\ell}}^\ast}}  -1)(\theta + 1/2)}
	\end{align*}
	hence 
	\[
	\inf\{ \gamma \geq0: C\,  n^{\gamma - (2^{{{\boldsymbol{\ell}}^\ast}}  -1)(\theta + 1/2)}\geq 1\}
	=(2^{{{\boldsymbol{\ell}}^\ast}}  -1)(\theta + 1/2)
	= \gamma^{\sharp}
	({\boldsymbol{\ell}},{\boldsymbol{c}}).
	\]
	Similar arguments imply that $\gamma^{\flat}
	({\boldsymbol{\ell}},{\boldsymbol{c}})=
	(2^{{{\boldsymbol{\ell}}^\ast}}  -1)(\theta + 1/2)$.
\end{remark}
	
\par

Let  $\Omega \subset [-1,1]^d$ be a measurable subset,
$p \in [1,\infty]$,
and $\alpha \in (0,\infty)$, for each measurable
$u : \Omega \to \mathbb{R}$, we define
\[
\Gamma_{\alpha,p}^{\varrho_2} (u)
:= \max
\Big\{
\| u \|_{L^p(\Omega)} , \;
\sup_{n \in \mathbb{N}}
\big[
n^{\alpha}
\cdot d_{p} \bigl(u,
\Sigma^{\boldsymbol{\ell}, \boldsymbol{c},\varrho_2}_n\bigr)
\big]
\Big\}
\in [0,\infty],
\]
where $d_{p}(u, \Sigma) := \inf_{g \in \Sigma} \| u - g \|_{L^p(\Omega)}$.

The remaining issue is that since the set
$\Sigma^{\boldsymbol{\ell}, \boldsymbol{c},\varrho_2}_n$ is in general neither
closed under addition nor under multiplication with scalars,
$\Gamma_{\alpha,p}^{\varrho_2}$
is \emph{not a (quasi)-norm}.
To resolve this issue,
taking inspiration from the theory of Orlicz spaces
(see e.g.~\cite[Theorem~3 in Section~3.2]{RaoRenOrliczSpaces}),
we define the \emph{neural network approximation space quasi-norm}
$
\| \cdot \|_{A^{\alpha,p}_{\boldsymbol{\ell},\boldsymbol{c},\varrho_2}}
$
as
\[
\| u \|_{A^{\alpha,p}_{\boldsymbol{\ell},\boldsymbol{c},\varrho_2}}
:= \inf \bigl\{ \theta > 0 \,\,\colon\,\,
\Gamma_{\alpha,p}^{\varrho_2}(u / \theta) \leq 1 \bigr\}
\in [0,\infty],
\]

giving rise to the \emph{approximation space}

\[
A^{\alpha,p}_{\boldsymbol{\ell},\boldsymbol{c},\varrho_2}
:= A^{\alpha,p}_{\boldsymbol{\ell},\boldsymbol{c},\varrho_2} (\Omega)
:= \bigl\{
u \in L^p(\Omega)
\,\,\colon\,\,
\| u \|_{A^{\alpha,p}_{\boldsymbol{\ell},\boldsymbol{c},\varrho_2}} < \infty
\bigr\}
.
\]

The following lemma summarizes the main elementary properties
of these spaces,
the proof can be concluded in a similar way to \cite[Appendix A.1.]{grohs2021},
hence it is left to the reader to check the details.

\begin{lemma}\label{lem:ApproximationSpaceProperties}
	Let $\varnothing \neq \Omega \subset [-1,1]^d$ be measurable,
	let $p \in [1,\infty]$ and $\alpha \in (0,\infty)$.
	Then,
	$A^{\alpha,p}_{\boldsymbol{\ell},\boldsymbol{c},\varrho_2}
	:= A^{\alpha,p}_{\boldsymbol{\ell},\boldsymbol{c},\varrho_2} (\Omega)$
	satisfies the following properties:
	\begin{enumerate}
		\item $(A^{\alpha,p}_{\boldsymbol{\ell},\boldsymbol{c},\varrho_2},
		\| \cdot \|_{A^{\alpha,p}_{\boldsymbol{\ell},\boldsymbol{c}}})$
		is a quasi-normed space.
		Precisely, given arbitrary measurable functions $u,v : \Omega \to \mathbb{R}$, then there exists  $C>0$ such that
		$$
			\| u + v \|_{A^{\alpha,p}_{\boldsymbol{\ell},\boldsymbol{c},\varrho_2}}
			\leq
			C \cdot (\| u\|
					_{A^{\alpha,p}_{\boldsymbol{\ell},\boldsymbol{c},\varrho_2}}
			+ \| v \|_{A^{\alpha,p}_{\boldsymbol{\ell},\boldsymbol{c},\varrho_2}}).
		$$
		
		\item We have $\Gamma_{\alpha,p}^{\varrho_2} (c u) \leq |c|
		\, \Gamma_{\alpha,p}^{\varrho_2}(u)$ for $c \in [-1,1]$.
		
		\item $\Gamma_{\alpha,p}^{\varrho_2}(u) \leq 1$
		if and only if $\| u \|_{A^{\alpha,p}_{\boldsymbol{\ell},\boldsymbol{c},\varrho_2}} \leq 1$.
		
		\item $\Gamma_{\alpha,p}^{\varrho_2}(u) < \infty$
		if and only if
		$\| u \|_{A^{\alpha,p}_{\boldsymbol{\ell},\boldsymbol{c},\varrho_2}} < \infty$.
		
		\item $A^{\alpha,p}_{\boldsymbol{\ell},\boldsymbol{c},\varrho_2} (\Omega)
		\hookrightarrow L^p(\Omega)$.
		Furthermore, if $\Omega \subset \overline{\Omega^\circ}$,
		then $A^{\alpha,\infty}_{\boldsymbol{\ell},\boldsymbol{c},\varrho_2}(\Omega)
		\hookrightarrow C_b(\Omega)$,
		where $C_b (\Omega)$ denotes the Banach space of continuous functions
		that are bounded and extend continuously to the closure $\overline{\Omega}$
		of $\Omega$.
	\end{enumerate}
\end{lemma}

\begin{proof}
	Proof is provided in Appendix \ref{appendix}.
\end{proof}

\subsection{Deterministic and Monte Carlo optimal orders}%
\label{sub:DeterministicAlgorithms}
In the current paper, we focus on the unit ball in the neural network
approximation space
$A
^{\alpha,\infty}_{\boldsymbol{\ell},\boldsymbol{c},\varrho_2}([0,1]^d)
$, i.e.,
\begin{equation}
	U_{\boldsymbol{\ell}, \boldsymbol{c},
		 \varrho_2}^{\alpha,\infty}([0,1]^d)
	:= \big\{u \in
	 A^{\alpha,\infty}_{\boldsymbol{\ell},\boldsymbol{c},\varrho_2}([0,1]^d)
	\,\,\colon\,\,
	\| u \|_{A^{\alpha,\infty}_{\boldsymbol{\ell},\boldsymbol{c},\varrho_2}}
	\leq 1
	\big\}.
	\label{eq:UnitBallDefinition}
\end{equation}

Let $\varnothing \neq U \subset C([0, 1]^d )$ be bounded
and $Y$ be a Banach space.
A deterministic method of order $m \in \mathbb{N}$
is a map $A : U \to Y$, written as $A \in \Alg_m (U,Y)$),
such that there exists a sample points
$X = (x_1,\dots,x_m) \in ([0,1]^d)^m$
and a map $Q : \mathbb{R}^m \to Y$ with the following property:

\[
A (u) = Q\bigl(u(x_1),\dots,u(x_m)\bigr)
\qquad \forall \, u \in U .
\]

The optimal error for deterministically approximating
$S : U \to Y$ using $m$ point samples is then

\[
e^{\deterministic}_m (U,S)
:= \inf_{A \in \Alg_m (U,Y)} \sup_{u \in U} \| A(u) - S(u) \|_Y .
\]

The optimal order for deterministically approximating
$S : U \to Y$ using point samples
is given by

\begin{equation}
	err^{\deterministic} (U,S)
	:= \sup
	\big\{
	\beta \geq 0
	\,\,\colon\,\,
	\exists \, C > 0 \,\,
	\forall \, m \in \mathbb{N}: \quad
	e^{\deterministic}_m (U,S) \leq C \cdot m^{-\beta}
	\big\} .
	\label{eq:OptimalOrderDeterministic}
\end{equation}

\par

A Monte Carlo method using $m \in \mathbb{N}$ point samples 
is a tuple $(\boldsymbol{A},\boldsymbol{m})$
where  $\boldsymbol{A}$ is a family of map
$\boldsymbol{A} = (A_\omega)_{\omega \in \Omega}$
and $A_\omega : U \to Y$
indexed by a probability space $(\Omega,\mathcal{F},\mathbb{P})$
and $\boldsymbol{m} : \Omega \to \mathbb{N}$, denoted by
$ \Alg^{\MonteCarlo}_m (U,Y)$,
such that
the following holds true
\begin{enumerate}
	\item \label{enu:MonteCarloNumberMeasurements}
	$\boldsymbol{m}$ is measurable and 
	$\mathbb{E} [\boldsymbol{m}] \leq m$;
	\item \label{enu:MonteCarloMeasurable}
	for any $u \in U$, the map $\omega \mapsto A_\omega(u)$
	is measurable with respect to the Borel $\sigma$-algebra on $Y$;
	\item for each $\omega \in \Omega$,
	we have $A_\omega \in \Alg_{\boldsymbol{m}(\omega)}(U,Y)$.	
\end{enumerate}

The optimal Monte Carlo error for approximating
$S : U \to Y$ using $m$ point samples is

\begin{equation*}
	e_m^{\MonteCarlo} (U, S)
	:= \inf_{(\boldsymbol{A},\boldsymbol{m}) \in \Alg^{\MonteCarlo}_m (U, Y)}
	\sup_{u \in U}
	\mathbb{E}_\omega \bigl[\| S(u) - A_\omega (u) \|_Y\bigr] .
\end{equation*}

The optimal Monte Carlo order for approximating $S : U \to Y$
using point samples is

\[
err^{\MonteCarlo} (U, S)
:= \sup
\big\{
\beta \geq 0
\,\,\colon\,\,
\exists \, C > 0 \,\,
\forall \, m \in \mathbb{N}: \quad
e_m^{\MonteCarlo} (U,S)
\leq  C \cdot m^{-\beta}
\big\} .
\]

The following lemma shows that if a hardness result holds
for deterministic algorithms \emph{in the average case},
then this implies a hardness result for Monte Carlo algorithms.
The proof can be found in \cite[Subsection 2.4]{grohs2021}.

\begin{lemma}\label{lem:MonteCarloHardnessThroughAverageCase}
	Let $\varnothing \neq U \subset C([0,1]^d)$ be bounded, let $Y$ be a Banach space,
	and let $S : U \to Y$.
	Assume that there exist $\lambda \in [0,\infty)$, $\kappa > 0$, and $m_0 \in \mathbb{N}$
	such that for every $m \in \mathbb{N}_{\geq m_0}$
	there exists a finite set $\Gamma_m \neq \varnothing$
	and a family of functions
	$(u_{\gamma})_{\gamma \in \Gamma_m} \subset U$ satisfying
	\begin{equation}
		\avsum_{\gamma \in \Gamma_m}
		\| S(u_\gamma) - A(u_\gamma) \|_Y
		\geq \kappa \cdot m^{-\lambda}
		\qquad \text{for any } \, A \in \Alg_m (U, Y) .
		\label{eq:AverageCaseHardness}
	\end{equation}
	Then $err^{\deterministic}(U,S),err^{\MonteCarlo}(U,S) \leq \lambda$.
\end{lemma}

\section{Properties of
	\texorpdfstring{ReQU approximation spaces
	$U_{\boldsymbol{\ell},\boldsymbol{c},\varrho_2}^{\alpha,\infty}
	([0,1]^d)$}
	{neural network approximation spaces}}
\label{sec:HatFunctionConstruction}

We start by introducing the so-called \emph{indicator-like function}
$\lambda_{M,y}^p : \mathbb{R} \rightarrow [0,1]$,
where $p\in\mathbb{N}$, $y\in \mathbb{R}$ and $M>0$, as follows:
\begin{equation*}
	\lambda_{M,y}^p(x) =
	\left\{ \begin{array}{ll}
		0 &\text{if} \quad x\leq y-1/M
		\\
		\left(1- (-M(x-y))^p \right)^p &\text{if} \quad y-1/M \leq x\leq y
		\\
		\left(1- (M(x-y))^p \right)^p&\text{if} \quad y \leq x \leq y+1/M
		\\
		0 &\text{if} \quad x\geq y+1/M.
	\end{array}\right.
\end{equation*}
The  function $M^{-p^2}\lambda_{M,y}^p$  can be implemented by two hidden layers 
RePU neural networks as follows:
\begin{equation}\label{eq:ind_p_rep}
	M^{-p^2}\lambda_{M,y}^p(x)
	= \varrho_p\Big(\frac{1}{M^p}- \big(\varrho_p( x-y) +\varrho_p(y-x)\big)\Big).
\end{equation}
Note that, if $p=1$,  then the indicator-like function
$\lambda_{M,y}^1(x)$
represents the \emph{hat function}.
Moreover, when $p=2$,  then  $\lambda_{M,y}^2(x)$ is a \emph{bump function}.
For $x=(x_1, \dots,x_d) \in\mathbb{R}^d$, we denote by
$\Lambda_{M,y}^p:\mathbb{R}^d \rightarrow \mathbb{R}, x \mapsto \lambda_{M,y}^p(x_1)$.

The construction of the function $\Lambda_{M,y}^2 : \mathbb{R}^d \to \mathbb{R}$
clearly shows that it has a controlled support with respect
to the first variable \(x_1\)
and unbounded support for the other variables
\(x_2, \dots, x_d\).
Similar to the construction in \cite{grohs2021}
we modify the indicator-like function in order to get a function
with controlled support with respect to all variables \(x_1, \dots, x_d\).
To that aim we compose the indicator-like function,
for \(p=2\), with a \emph{Heaviside-like} function.
It is worth mentioning that all these functions can be implemented
by ReQU neural networks.

\begin{lemma}\label{lem:MultidimensionalHatProperties}
	Given $d \in \mathbb{N}$, $M > 0$ and $y \in \mathbb{R}^d$, define
	\[
	\begin{alignedat}{5}
		&& \theta : \quad
		& \mathbb{R} \to [0,1], \quad
		&& x \mapsto 2(\varrho_2(x) - 2\varrho_2(x-\tfrac{1}{2}) + \varrho_2(x-1)) ,
		\\[1ex]
		&& \Delta_{M,y} : \quad
		& \mathbb{R}^d \to \mathbb{R}, \quad
		&& x \mapsto \bigg[ \sum_{j=1}^d \lambda_{M,y_j}^2(x_j) \bigg] - (d - 1) ,
		\\[1ex]
		\quad \text{ and } \quad
		&& \vartheta_{M,y} : \quad
		& \mathbb{R}^d \to [0,1], \quad
		&& x \mapsto \theta\bigl(\Delta_{M,y}(x)\bigr) .
	\end{alignedat}
	\]
	Then the function $\vartheta_{M,y}$ has the following properties:
	\begin{enumerate}[label=\arabic*)]
		\item $\vartheta_{M,y}(x) = 0$ for all $x \in \mathbb{R}^d
		\setminus \bigl(y + M^{-1} (-1,1)^d\bigr)$;
		\item $\| \vartheta_{M,y} \|_{L^p (\mathbb{R}^d)}
		\leq (2 / M)^{d/p}$ for arbitrary $p \in (0,\infty]$;
		\item For any $p \in (0,\infty]$ there is a constant
		$C = C(d,p) > 0$ satisfying
		
		\[
		\| \vartheta_{M,y} \|_{L^p([0,1]^d)}
		\geq C \cdot M^{-d/p},
		\quad
		\text{for all } \, y \in [0,1]^d \text{ and } M \geq \tfrac{1}{2d} .
		\]
	\end{enumerate}
\end{lemma}

\begin{proof}
	The first and second points in the lemma can be deduced in a similar way
	as in the proof of \cite[Lemma 3.4]{grohs2021},
	hence we left the details for the reader to check.
	It remains to show the last assertion in the lemma.
	Therefore, we let
	$T := \frac{1}{2 d M} \in (0,1]$ and $P := y + [-T,T]^d$.
	For $x \in P$ and arbitrary $j \in \underline{d}$,
	we have $|x_j - y_j| \leq\frac{1}{2 d M}$.
	Using the fact that $\lambda_{M,y_j}^2$ is continuously differentiable
	function with bounded derivative such that
	\begin{equation*}
		\left(\lambda_{M,y_j}^2\right)'(x_1) =
		\left\{ \begin{array}{ll}
			-4M^2(x_1-y)\bigl(1-M^2(x_1-y)^2\bigr)&\text{if}
			\quad y- 1/M \leq x_1 \leq y +1/M
			\\[1ex]
			0 &\text{ otherwise}.
		\end{array}\right.
	\end{equation*}

	Moreover,	
	\(\bigl|\left(\lambda_{M,y_j}^2\right)'\bigr|\)
	attends its maximum at
	\( y\pm \tfrac{1}{\sqrt{3}M} \) with value equal to
	\(\frac{8M}{3\sqrt{3}}\). 
	Then, it follows that
	$\lambda_{M,y_j}^2$ is a Lipschitz function with
	$\operatorname{Lip}(\lambda_{M,y_j}^2) \leq \frac{8M}{3\sqrt{3}}$.
	Furthermore,  $\lambda_{M,y_j}^2(y_j) = 1$, hence we get
	\[
	\lambda_{M,y_j}^2(x_j)
	\geq \lambda_{M,y_j}^2(y_j) -
	\bigl|\lambda_{M,y_j}^2(y_j) -\lambda_{M,y_j}(x_j)\bigr|
	\geq 1 - \frac{8M}{3\sqrt{3}} \cdot \frac{1}{2 d M}
	=    1 - \frac{4}{3\sqrt{3} d} .
	\]
	The fact that this holds for all $j \in \underline{d}$, implies that
	\[{
		\Delta_{M,y}(x)
		= \sum_{j=1}^d \Lambda_{M,y_j}(x_j) - (d \!-\! 1)
		\geq d \!\cdot\! (1 - \frac{4}{3\sqrt{3} d}) - (d \!-\! 1)
		= 1 - \frac{4}{3\sqrt{3} },
	}\]
	and since $\theta$ is non-decreasing we get 
	$$
	\vartheta_{M,y}(x) = \theta(\Delta_{M,y}(x))
	\geq\theta(1 - \frac{4}{3\sqrt{3} }) = \frac{2}{81}(9 - 4 \sqrt(3))^2.
	$$
	All in all, using the previous estimate and \cite[Lemma A.2]{grohs2021},
	for any Lebesgue measure $\boldsymbol{\mu}$,
	it follows that
	$$
	\boldsymbol{\mu}([0,1]^d \cap (x+[-T,T]^d)) \geq 2^{-d} T^d
	\geq C_1 \cdot M^{-d} \text{ such that }C_1= C_1(d) > 0.
	$$
	
	Finally, we get
	\[
	\| \vartheta_{M,y} \|_{L^p([0,1]^d)}
	\geq \frac{2}{81}(9 - 4 \sqrt(3))^2
	[\boldsymbol{\mu}([0,1]^d \cap (x+[-T,T]^d))]^{1/p}
	\geq C^{1/p} M^{-d/p}.
	\]
\end{proof}
	
In the next result we show that $\vartheta_{M,y}$
can be implemented by ReQU neural networks.
	
\begin{lemma}\label{lem:MultidimensionalHatImplementation}
	Let ${\boldsymbol{\ell}} : \mathbb{N} \to \mathbb{N}_{\geq 5} \cup \{ \infty \}$
	and ${\boldsymbol{c}} : \mathbb{N} \to \mathbb{N} \cup \{ \infty \}$ be non-decreasing.
	Let $n \in \mathbb{N}$, $M\geq 1$ and
	$1\leq C \leq C^8 \leq {\boldsymbol{c}}(n)$,
	as well as $L \in \mathbb{N}_{\geq 5}$ with $L \leq {\boldsymbol{\ell}}(n)$.
	Then
	\[\frac{C^{2^{L}-1}n^{(2^{L}-1)/2}}{4M^8}\cdot \vartheta_{M,y}
	\in \Sigma^{\boldsymbol{\ell}, \boldsymbol{c},\varrho_2}_{16n^8d  + 7L},
	\quad \text{ for any } \, y \in [0,1]^d .
	\]
\end{lemma}

\begin{proof}
	Let $y \in [0,1]^d$ be fixed.
	For $j \in \underline{d}$, denote by $e_{j, d} $ the $j$-th standard basis vector
	in \(\mathbb{R}^{d \times 1}\).
	Let \(\zeta := \sqrt{\frac{d-1}{d}}\) and 
	\begin{align*}
		A_1
		& := \Big(
		\underbrace{e_{1, d} \big| -e_{1, d}, \big| \dots
			\big| e_{1,d} \big| -e_{1, d} }_{2n \text{ times}},
		\dots,\;
		\underbrace{e_{d, d} \big| -e_{d,d} \big| \dots
			\big| e_{d,d}\big| -e_{d,d}}_{2n \text{ times}}
		\Big)^T \in \mathbb{R}^{2 n d \times d},
		\\[1ex]
		b_1
		& :=
		\Big(
		\underbrace{-y_1, y_1  , \dots , -y_1, y_1}_{2n \text{ times}},
		\;\dots,\;
		\underbrace{-y_d, y_d   , \dots , -y_d, y_d  }_{2n \text{ times}}
		\Big)^T \in \mathbb{R}^{2 n d},
	\end{align*}

	\begin{align*}
		A_{2}
		& := \left(
		\begin{matrix}
			-1&	-1& 0& 0& \dots& 0&0
			\\
			 0&	 0&	0&	0& \dots&  0& 0
			\\
			0&0 &0&0&\dots& 0& 0
			\\
			0& 0& -1& -1& \dots&0 & 0
			\\
			0& 0& 0& 0& \dots&0 & 0
			\\
			0& 0& 0& 0& \dots&0 & 0
			\\
			\vdots&\vdots&\vdots&\vdots&\vdots&\vdots&\vdots
			\\
			0& 0& \dots&0& 0&-1& -1
			\\
			0& 0& \dots&0& 0&0& 0
			\\
			0& 0& \dots&0& 0&0& 0
		\end{matrix}
		\right)
		\in \mathbb{R}^{3n^8d\times 2 n d},
		\\[1ex]
		b_2
		&:=\frac{1}{M^2}\cdot
		\Big(		1, \zeta, \frac{C^8M^2}{\sqrt{d}},
		1,\zeta, \frac{C^8M^2}{\sqrt{d}},
		\dots,
		1, \zeta,\frac{C^8M^2}{\sqrt{d}}
		\Big)^T \in \mathbb{R}^{3n^8d\times 1},
	\end{align*}

	\begin{align*}
		A_3
		&:=\frac{1}{n^8} \cdot \left(
		\begin{matrix}
			1& -1& 0&1& -1& 0 &\dots&  1& -1& 0
			\\[1ex]
			1&-1& 0&1&-1& 0 &\dots&  1& -1& 0
			\\[1ex]
			1&-1& 0&1&-1& 0 &\dots&  1& -1& 0
			\\[1ex]
			0& 0& n^8& 0& 0& n^8 &\dots&  0& 0& n^8
		\end{matrix}
		\right) \in \mathbb{R}^{4\times 3n^8d}, 
		\\[1ex]
		b_3
		&:= \frac{1}{M^4}\cdot
		\left(	0,	-\frac{1}{2},-1,0	\right)^T \in \mathbb{R}^{4\times 1},
	\end{align*}

	\begin{align*}
		D_1
		:=\frac{1}{2} \cdot \Bigg(
		\begin{matrix}
			1&-2& 1& 2/C\sqrt{n}
			\\[1ex]
			-1 & 2 &-1 & 2/C\sqrt{n}
		\end{matrix}
		\Bigg) \in \mathbb{R}^{2\times 4}, 
		\qquad
		D_2
		:=\frac{1}{2} \cdot \left(
		\begin{matrix}
			1&-2& 1& 0
			\\[1ex]
			-1&2& -1& 0
			\\[1ex]
			0 & 0 &0 & 2
		\end{matrix}
		\right) \in \mathbb{R}^{3\times 4},
	\end{align*}

	\begin{align*}
		A
		&:=\frac{1}{4} \cdot \Bigg(
		\begin{matrix}
			1&-1& 0
			\\
			-1& 1& 0
			\\
			0& 0& 4
		\end{matrix}
		\Bigg) \in \mathbb{R}^{3\times 3},
		\qquad
		&&
		\alpha
		:= 
		\left(	1,	1,	0\right)^T \in \mathbb{R}^{3\times 1},
		\\[1ex]
		K
		&:=\frac{1}{4} \cdot \Bigg(
		\begin{matrix}
			1&-1& {4}/{C\sqrt{n}}
			\\[1ex]
			-1& 1& {4}/{C\sqrt{n}}
		\end{matrix}
		\Bigg) \in \mathbb{R}^{2\times 3},
		\qquad
		&&E
		:= \frac{1}{4}\Big(1,-1\Big) \in R^{1\times 2}.
	\end{align*}
	
	For any \(j\in \{1,\dots, d\}\), in view of \eqref{eq:ind_p_rep}
	where $p=2$,
	we can represent the function \(\lambda _{M, y_j}^2 (x_j)\)
	through ReQU network
	\[
	\frac{1}{M^4}\cdot\lambda_{M,y_j}^2 (x_j)
	= 
	\varrho_2
	\Big(
	\frac1{M^2}- 
	\big(
	\varrho_2(x_j-y_j)+\varrho_2(y_j-x_j)
	\big)
	\Big)\text{ for any } x_j, y_j\in \mathbb{R}.
	\]
	Consequently, for the realization of the architecture
	\(\phi_3 = \left( (A_1, b_1),(A_2, b_2),(A_3, b_3)\right)\)
	with respect to the ReQU activation function, assuming that 
	$\eta_j = x_j -y_j$ for any $j\in \{1, \dots, d\}$,
	we have
	\begin{equation*}
		A_1x+b_1= \Bigg(\underbrace{\eta_1, -\eta_1, 
			\dots
			, \eta_1, -\eta_1}_{2n\text{ times}},
		\dots, 
		\underbrace{\eta_d, -\eta_d,\dots, \eta_d, -\eta_d}_{2n \text{ times}}\Bigg)^T,
	\end{equation*}
	and
	\begin{multline*}
		A_2\varrho_2\left(A_1x+b_1\right) + b_2=
		\\
		\Bigg(\underbrace
		{-\varrho_2(\eta_1)-\varrho_2(-\eta_1)+\frac{1}{M^2},
			\frac{\zeta}{M^2},\frac{C^8}{\sqrt{d}},
			\dots,
			-\varrho_2(\eta_1)-\varrho_2(-\eta_1)+
			\frac{1}{M^2},
			\frac{\zeta}{M^2},\frac{C^8}{\sqrt{d}}}_{3n^8\text{ times}},
		\\[1ex]\dots,\\
		\underbrace
		{-\varrho_2(\eta_d)-\varrho_2(-\eta_d)+
			\frac{1}{M^2},
			\frac{\zeta}{M^2},\frac{C^8}{\sqrt{d}},
			\dots, -\varrho_2(\eta_d)-\varrho_2(-\eta_d)+
			\frac{1}{M^2},
			\frac{\zeta}{M^2},\frac{C^8}{\sqrt{d}}}_{3n^8 \text{ times}}\Bigg)^T.
	\end{multline*}
	
	\begin{align*}
		\Big(R_{\varrho_2}\phi_3(x)\Big)_1
		&= \Bigl(A_3\varrho_2\bigl(A_2 \, \varrho_2(A_1 x + b_1) + b_2\bigr)+b_3\Bigr)_1
		\\
		& = \frac{1}{n^8}\cdot
		\sum_{j=1}^d
		\sum_{\ell=1}^{n^8}
		\Big[
		\varrho_2
		\Big(\frac{1}{M^2}- 
		\left(
		\varrho_2 \bigl( \langle x, e_j \rangle - y_j \bigr)
		+ \varrho_2 \bigl( -\langle x, e_j \rangle + y_j )
		\right)\Big)
		- \varrho_2(\frac{\zeta}{M^2})
		\Big] \\
		& = \frac{1}{M^4}\cdot
		\sum_{j=1}^d
		\Big[
		\lambda_{M,y_j}^2 (x_j)
		- \frac{d-1}{d}
		\Big]
		= \frac{1}{M^4}\cdot\Delta_{M,y} (x) .
	\end{align*}
	In the same way, we get
	\begin{align*}
		\Big(R_{\varrho_2}\phi_3(x)\Big)_2
		&=
		\frac{1}{M^4}\cdot \left(\Delta_{M,y}(x)-\tfrac{1}{2}\right)
		\\[1ex] 
		\Big(R_{\varrho_2}\phi_3(x)\Big)_3
		&=
		\frac{1}{M^4}\cdot \left(\Delta_{M,y}(x) - 1\right)
		\intertext{and}
		\Big(R_{\varrho_2}\phi_3(x)\Big)_4
		&=
		C^{16}n^8= C^{2^4}n^{2^3}.
	\end{align*}
	Let $\psi_1 = \Big( (A_1, b_1),(A_2, b_2),(A_3, b_3), (D_1,0)\Big)$,
	a straightforward computation shows that
	\begin{align*}
		R_{\varrho_2}\psi_1(x)&=
		\frac{1}{2} \cdot \Bigg(
		\begin{matrix}
			1&-2& 1& 2/C\sqrt{n}
			\\[1ex]
			-1 & 2 &-1 & 2/C\sqrt{n}
		\end{matrix}
		\Bigg)
		\left(
		\begin{matrix}
			\frac{1}{M^8}\cdot \varrho_2\left(\Delta_{M,y}(x)\right)
			\\[1ex]
			\frac{1}{M^8}\cdot\varrho_2\left(\Delta_{M,y}(x) - \frac12\right)
			\\[1ex]
			\frac{1}{M^8}\cdot \varrho_2\left(\Delta_{M,y}(x) - 1\right)
			\\[1ex]
			C^{2^5}n^{2^4}
		\end{matrix}
		\right) ,
	\end{align*}
	hence, we have
	\begin{align*}
		\Big(R_{\varrho_2}\psi_1(x)\Big)_1
		&=
		\frac{1}{2M^8}\cdot
		\Big(
		\varrho_2\bigl( \Delta_{M,y}(x)\bigr)
		-2\varrho_2\bigl(\Delta_{M,y}(x) - \tfrac{1}{2}\bigr)
		+ \varrho_2\bigl(\Delta_{M,y}(x) - 1\bigr) 
		\Big)+ \frac{C^{2^5}n^{2^4}}{C\sqrt{n}}
		\\[1ex]
		&= \frac{1}{4M^8}\cdot \theta(\Delta_{M,y}(x))+ C^{2^5-1}n^{2^4-\frac 12}
		=\dfrac{\vartheta_{M,y}(x)}{4M^8} +C^{2^5-1}n^{(2^5-1)/2}
		\intertext{and}
		\Big(R_{\varrho_2}\psi_1(x)\Big)_2
		&= -\dfrac{\vartheta_{M,y}(x)}{4M^8}+ C^{2^5-1}n^{(2^5-1)/2}.
	\end{align*}
	Then, 
	\[
	R_{\varrho_2}\psi_1(x)
	=
	\left(
	\begin{matrix}
		\dfrac{\vartheta_{M,y}(x)}{4M^8} + C^{2^5-1}n^{(2^5-1)/2}
		\\[2ex]
		-\dfrac{\vartheta_{M,y}(x)}{4M^8}+ C^{2^5-1}n^{(2^5-1)/2}
	\end{matrix}
	\right).
	\]
	In  a similar way, we let 
	$\psi_2 = \Big( (A_1, b_1),(A_2, b_2),(A_3, b_3), (D_2,\alpha)\Big)$.
	We then get
	\begin{equation}\label{eq:psi_2Realization}
		R_{\varrho_2}\psi_2(x)
		=
		\left(
		\begin{matrix}
			\dfrac{\vartheta_{M,y}(x)}{4M^8}+1
			\\[2ex]
			-\dfrac{\vartheta_{M,y}(x)}{4M^8} +1
			\\[2ex]
			C^{2^5}n^{2^4}
		\end{matrix}
		\right).
	\end{equation}

	\par
	
	Since $y \in [0,1]^d, n\in \mathbb{N}$, $\zeta\leq 1\leq M$ and that
	$1\leq C\leq C^8\leq \boldsymbol{c}(n)$
	it is clear that the weights 
	\[
	\| A_1 \|_{\infty},
	\| A_2 \|_{\infty},
	\| A_3 \|_{\infty},
	\| D_1 \|_{\infty},
	\| D_2 \|_{\infty},
	\| A \|_{\infty},
	\| K \|_{\infty},
	\| E \|_{\infty}
	\leq \boldsymbol{c}(n)
	\]
	and the biases
	\[
	\|b_1  \|_{\infty},
	\| b_2 \|_{\infty},
	\| b_3 \|_{\infty},
	\| \alpha \|_{\infty},
	\leq \boldsymbol{c}(n).
	\]
	Furthermore, we have
	$$
	\| A_1 \|_{\ell^0} \leq 2 nd,\;
	\| A_2 \|_{\ell^0} \leq 2 n^8d,\;	
	\| A_3 \|_{\ell^0} \leq 3\times 2n^8d + n^8d=7n^8d,
	$$
	and
	\begin{align*}	
		\| D_1 \|_{\ell^0}&= 8,\; \| D_2 \|_{\ell^0}= 7,
		\\[1ex]
		\| A \|_{\ell^0}&=5, \;\| K \|_{\ell^0}=6,\;
		\| E \|_{\ell^0} =2.
	\end{align*}
	Moreover, the biases satisfy
	\begin{align*}
		\| b_1 \|_{\ell^0} &\leq 2 nd,\; \| b_2 \|_{\ell^0} \leq 3n^8d
		\\[1ex]
		\| b_3 \|_{\ell^0} &=  \| \alpha \|_{\ell^0}=2.
	\end{align*}
	In order to prove the claim of the lemma, we distinguish two cases
	depending on the depth of the network.

	\textbf{Case 1:} $L=5$.
	Here we let  $\Phi$ the chosen architecture 
	such that 
	$$
	\Phi:= \big( (A_1,b_1), (A_2,b_2),(A_3, b_3), (D_1,0), (E,0) \big)
	$$
	The ReQU realization of $\Phi$
	shows that	
	\begin{align*}
		R_{\varrho_2} \Phi (x)
		&= \frac 14 \big(1, -1\big)R_{\varrho_2} \psi_1 (x)
		\\[1ex]
		&=\frac 14 \big(1, -1\big)
		\left(
		\begin{matrix}
			\varrho_2\left(\dfrac{\vartheta_{M,y}(x)}{4M^8} +C^{2^5-1}n^{(2^5-1)/2}\right)
			\\[2ex]
			\varrho_2\left(-\dfrac{\vartheta_{M,y}(x)}{4M^8}+C^{2^5-1}n^{(2^5-1)/2}\right)
		\end{matrix}
		\right)
		\intertext{since for any $x$, and $n, M,C\geq 1$
			it holds that $0\leq \dfrac{\vartheta_{M,y}(x)}{4M^8} \leq C^{2^5-1}n^{(2^5-1)/2}$, we get
		}
		R_{\varrho_2} \Phi (x)
		&=\frac{C^{2^5-1}n^{(2^5-1)/2}}{4M^8}\cdot\vartheta_{M,y}(x)
		=\frac{C^{2^L-1}n^{(2^L-1)/2}}{4M^8}\cdot\vartheta_{M,y}(x).
	\end{align*}
	
	The monotonicity of the functions ${\boldsymbol{c}}$
	and ${\boldsymbol{\ell}}$ and the fact that $L=5$
	imply  the complexity of the network $\Phi$, indeed
	\begin{align*}
		W(\Phi) &= \sum_{k=1}^{3}
		\left(\| A_k \|_{\ell^0} + \| b_k \|_{\ell^0}\right)
		+ \| D_1 \|_{\ell^0} +	\| E \|_{\ell^0}
		\\
		&\leq 2nd +2nd  + 2n^8d+3n^8d  + 7n^8d +2
		+8 + 2
		\\
		&\leq 12n^8d + 4nd +12 \leq 16n^8d  + 7L.
		\\[1ex]
		L(\Phi)&= L\leq \boldsymbol{\ell}(n) \leq \boldsymbol{\ell}(16n^8d  + 7L),
		\\[1ex]
		\left\Vert\Phi\right\Vert_{\mathcal{NN}}& \leq C^8 \leq\boldsymbol{c}(n) \leq \boldsymbol{c}(16n^8d  + 7L).
	\end{align*}

	We conclude that
	$\frac{C^{2^L-1}n^{(2^L-1)/2}}{4M^8}\cdot\vartheta_{M,y} \in
	\Sigma^{\boldsymbol{\ell}, \boldsymbol{c},\varrho_2}_{16n^8d  + 7L}$.

	\textbf{Case 2:} $L \geq 6$.
	In this case, we recall that $\psi_2  = \left((A_1, b_1),
	(A_2, b_2),
	(A_3, b_3),
	(D_2, \alpha)\right)$ and let the architecture $\Phi$ given by
	\[
	\Phi
	= \Big(\psi_2,
	\underbrace{(A, \alpha),\dots,(A, \alpha)}_{L-6 \text{ times}},(K,0),(E, 0).
	\Big)
	\]
	
	The fact that
	$0\leq \dfrac{\vartheta_{M,y}(x)}{4M^8}\leq 1$
	implies that 
	\begin{equation}\label{eq:RequIdentity}
		\frac14\left(
		\varrho_2\left(\dfrac{\vartheta_{M,y}(x)}{4M^8}+1\right)
		-
		\varrho_2\left(-\dfrac{\vartheta_{M,y}(x)}{4M^8}+1\right)
		\right)
		= \dfrac{\vartheta_{M,y}(x)}{4M^8},
	\end{equation}
	Then, from \eqref{eq:psi_2Realization} and \eqref{eq:RequIdentity},
	we get
	\begin{align*}
		A\varrho_2\left(R_{\varrho_2}\psi_2(x)\right)+\alpha
		&=\frac{1}{4} \cdot \Bigg(
		\begin{matrix}
			1&-1& 0
			\\
			-1& 1& 0
			\\
			0& 0& 4
		\end{matrix}
		\Bigg) 
		\left(
		\begin{matrix}
			\varrho_2\left(\dfrac{\vartheta_{M,y}(x)}{4M^8}+1\right)
			\\[2ex]
			\varrho_2\left(-\dfrac{\vartheta_{M,y}(x)}{4M^8} +1\right)
			\\[2ex]
			\varrho_2\left(C^{2^5}n^{2^4}\right)
		\end{matrix}
		\right)+
		\left(
		\begin{matrix}
			1
			\\[1ex]
			1
			\\[1ex]
			0
		\end{matrix} 
		\right)
		\\[2ex]
		&=\frac 14
		\left(
		\begin{matrix}
			\varrho_2\left(\dfrac{\vartheta_{M,y}(x)}{4M^8}+1\right)
			-\varrho_2\left(-\dfrac{\vartheta_{M,y}(x)}{4M^8} +1\right)
			\\[2ex]
			-\varrho_2\left(\dfrac{\vartheta_{M,y}(x)}{4M^8}+1\right)
			+
			\varrho_2\left(-\dfrac{\vartheta_{M,y}(x)}{4M^8} +1\right)
			\\[2ex]
			4\,C^{2^6}n^{2^5}
		\end{matrix}
		\right)
		+
		\left(
		\begin{matrix}
			1
			\\[1ex]
			1
			\\[1ex]
			0
		\end{matrix} 
		\right)
		\\[2ex]
		&=
		\left(
		\begin{matrix}
			\dfrac{\vartheta_{M,y}(x)}{4M^8}
			\\[2ex]
			-\dfrac{\vartheta_{M,y}(x)}{4M^8}
			\\[2ex]
			C^{2^6}n^{2^5}
		\end{matrix}
		\right)
		+
		\left(
		\begin{matrix}
			1
			\\[1ex]
			1
			\\[1ex]
			0
		\end{matrix} 
		\right)=
		\left(
		\begin{matrix}
			\dfrac{\vartheta_{M,y}(x)}{4M^8}+1
			\\[2ex]
			-\dfrac{\vartheta_{M,y}(x)}{4M^8}+1
			\\[2ex]
			C^{2^6}n^{2^5}
		\end{matrix}
		\right).
	\end{align*}
	We define, for any $j\in \{0,1,2,\dots\}$,
	the architecture $\Phi_j$ as 
	\[
	\Phi_j :=
	\begin{cases}
		\Big(\psi_2, \underbrace{
			(A, \alpha), \dots, (A, \alpha)}_{j\text{ times}}\Big), \quad&\text{when } j\in \mathbb{N},	
		\\
		\psi_2, \quad&\text{when } j=0.	
	\end{cases} 
	\]
	Thus, from the previous computations,
	we get
	\[
	R_{\varrho_2}\Phi_j(x)=
	\left(
	\begin{matrix}
		\dfrac{\vartheta_{M,y}(x)}{4M^8}+1
		\\[2ex]
		-\dfrac{\vartheta_{M,y}(x)}{4M^8}+1
		\\[2ex]
		C^{2^{j+5}}n^{2^{j+4}}
	\end{matrix}
	\right).
	\]
	
	Applying \(K \,\varrho_2(\bullet)\) to the previous 
	equation, we get
	\begin{align*}
		K \varrho_2&\left(R_{\varrho_2}\Phi_j(x)\right)
		= \frac{1}{4} \cdot \Bigg(
		\begin{matrix}
			1&-1& {4}/{C\sqrt{n}}
			\\[1ex]
			-1& 1& {4}/{C\sqrt{n}}
		\end{matrix}
		\Bigg) 
		\left(
		\begin{matrix}
			\varrho_2\left(\dfrac{\vartheta_{M,y}(x)}{4M^8}+1\right)
			\\[2ex]
			\varrho_2\left(-\dfrac{\vartheta_{M,y}(x)}{4M^8}+1\right)
			\\[2ex]
			\varrho_2\left(C^{2^{j+5}}n^{2^{j+4}}\right)
		\end{matrix}
		\right)
		\\[2ex]
		&= \left(
		\begin{matrix}
			\frac 14\left(\varrho_2\left(\dfrac{\vartheta_{M,y}(x)}{4M^8}+1\right)
			-\varrho_2\left(-\dfrac{\vartheta_{M,y}(x)}{4M^8}+1\right)\right)
			+C^{2^{j+6}}n^{2^{j+5}}/C\sqrt{n}
			\\[2ex]
			-\frac14\left(\varrho_2\left(\dfrac{\vartheta_{M,y}(x)}{4M^8}+1\right)
			-\varrho_2\left(-\dfrac{\vartheta_{M,y}(x)}{4M^8}+1\right)\right)
			+ C^{2^{j+6}}n^{2^{j+5}}/C\sqrt{n}
		\end{matrix}
		\right)
		\\[2ex]
		&=
		\left(
		\begin{matrix}
			\dfrac{\vartheta_{M,y}(x)}{4M^8}
			+ C^{2^{j+6}}n^{2^{j+5}}/C\sqrt{n}
			\\[2ex]
			-\dfrac{\vartheta_{M,y}(x)}{4M^8}
			+ C^{2^{j+6}}n^{2^{j+5}}/C\sqrt{n}
		\end{matrix}
		\right)
		=
		\left(
		\begin{matrix}
			\dfrac{\vartheta_{M,y}(x)}{4M^8}
			+ C^{2^{j+6}-1}n^{(2^{j+6}-1)/2}
			\\[2ex]
			-\dfrac{\vartheta_{M,y}(x)}{4M^8}
			+ C^{2^{j+6}-1}n^{(2^{j+6}-1)/2}
		\end{matrix}
		\right).
	\end{align*}
	
	Finally, we have
	\begin{align*}
		R_{\varrho_2}\Phi(x)&= E\varrho_2\left(
		K \varrho_2\left(R_{\varrho_2}\Phi_{L-6}(x)\right)
		\right)
		=\frac{1}{4}\Big(1,-1\Big) 
		\left(
		\begin{matrix}
			\varrho_2\left(\dfrac{\vartheta_{M,y}(x)}{4M^8}
			+ C^{2^{L}-1}n^{(2^{L}-1)/2}\right)
			\\[1ex]
			\varrho_2\left(-\dfrac{\vartheta_{M,y}(x)}{4M^8}
			+C^{2^{L}-1}n^{(2^{L}-1)/2}\right)
		\end{matrix}
		\right)
		\\[1ex]
		&=\frac 14\left( 
		\varrho_2\left(\dfrac{\vartheta_{M,y}(x)}{4M^8}
		+ C^{2^{L}-1}n^{(2^{L}-1)/2}\right)
		-
		\varrho_2\left(-\dfrac{\vartheta_{M,y}(x)}{4M^8}
		+ C^{2^{L}-1}n^{(2^{L}-1)/2}\right)
		\right)
		\\[1ex]
		&=\frac{C^{2^{L}-1}n^{(2^{L}-1)/2}}{4M^8}\cdot \vartheta_{M,y}(x).
	\end{align*} 
	
	\par
	
	As in the previous case, we determine the complexity of the network
	$\Phi$ when $L\geq 6$
	\begin{align*}
		W(\Phi) &=
		\sum_{k=1}^{3}
		\left(\| A_k \|_{\ell^0} + \| b_k \|_{\ell^0}\right)
		+
		\| D_2 \|_{\ell^0} +\| \alpha \|_{\ell^0}+
		(L-6)\left(\| A \|_{\ell^0} + \| \alpha \|_{\ell^0} \right)+
		\| E \|_{\ell^0}
		\\[1ex]
		&\leq2nd +2nd  + 2n^8d+3n^8d  + 7n^8d +2 +7 +2 +
		(L-6 )(5 +2) + 2
		\\[1ex]
		&\leq 4nd+ 12n^8d + 13+ (L-6)7 \leq 12n^8d + 4nd + 7L -29
		\\[1ex]
		&\leq 16n^8d  + 7L.
		\\[1ex]
		L(\phi)&= L\leq \boldsymbol{\ell}(n) \leq \boldsymbol{\ell}(16n^8d  + 7L),
		\\[1ex]
		\left\Vert\Phi\right\Vert_{\mathcal{NN}}& \leq C^8 \leq\boldsymbol{c}(n) \leq \boldsymbol{c}(16n^8d  + 7L).
	\end{align*}

	\par
	
	Consequently, we see that
	$\frac{C^{2^{L}-1}n^{(2^{L}-1)/2}}{4M^8} \cdot \vartheta_{M,y} \in \Sigma^{\boldsymbol{\ell}, \boldsymbol{c},\varrho_2}_{16n^8d  + 7L}$.	
\end{proof}	
	
\begin{lemma}\label{lem:MultidimensionalHatInSpace}
	Let ${\boldsymbol{\ell}},{\boldsymbol{c}} : \mathbb{N} \to \mathbb{N} \cup \{ \infty \}$
	be non-decreasing
	with ${\boldsymbol{\ell}}^\ast \geq 5$.
	Let $d \in \mathbb{N}$, $\alpha \in (0,\infty)$,
	and $0 < \gamma < \gamma^{\flat}({\boldsymbol{\ell}},{\boldsymbol{c}})$.
	Then there exists a constant
	$\kappa = \kappa(\gamma,\alpha,d,{\boldsymbol{\ell}},{\boldsymbol{c}}) > 0$
	such that for any $M \geq1$,
	we have
	\[
	g_{M,y}
	:= \kappa \cdot M^{-64\alpha / (8\alpha + \gamma)} \, \vartheta_{M,y}
	\in A_{\boldsymbol{\ell}, \boldsymbol{c},\varrho_2}^{\alpha,\infty}([0,1]^d)
	\qquad \text{with} \qquad
	\big\| g_{M,y} \big\|_
	{A_{\boldsymbol{\ell}, \boldsymbol{c},\varrho_2}^{\alpha,\infty}([0,1]^d)}
	\leq 1.
	\]
\end{lemma}

\begin{proof}
	From the fact that
	$\gamma < \gamma^{\flat}({\boldsymbol{\ell}},{\boldsymbol{c}})$,
	it follows that there exist
	$L = L(\gamma,{\boldsymbol{\ell}},{\boldsymbol{c}})
	\in \mathbb{N}_{\leq {\boldsymbol{\ell}}^\ast}$
	and $C_1 = C_1(\gamma,{\boldsymbol{\ell}},{\boldsymbol{c}}) > 0$ where
	$n^{\gamma} \leq C_1  \cdot
	\boldsymbol{c}(n)^{2^{{L}}-1}\cdot	
	n^{(2^{{L}}  -1)/2}$
	for any $n \in \mathbb{N}$.
	Without loss of generality, in view of the fact that 
	${\boldsymbol{\ell}}^\ast\geq 5$, we assume that $L \geq 5$.
	Using \Cref{eq:MaximalDepth,eq:MaximalCoef,eq:GammaDefinition},
	the fact that 
	$L \leq {\boldsymbol{\ell}}^\ast$
	implies the existence of 
	$n_0 \in \mathbb{N}$ depends on $\gamma,{\boldsymbol{\ell}}$ and ${\boldsymbol{c}} $
	such that $L \leq {\boldsymbol{\ell}}(n_0)$.
 
	Furthermore, let \( M\geq 1\) and 
	set $n := n_0 \cdot \big\lceil M^{8/(8\alpha+\gamma)} \big\rceil$,
	noting that $n \geq n_0$.
	Since
	$$
	n^{\gamma} \leq C_1 \cdot
	\boldsymbol{c}(n)^{2^{{L}}-1}\cdot	
	n^{(2^{{L}}  -1)/2},
	$$
	there exists $1 \leq C \leq C^8 \leq {\boldsymbol{c}}(n)$ such that
	$$
	n^{\gamma} \leq C_1  \cdot
	C^{2^{{L}}-1}\cdot	
	n^{(2^{{L}}  -1)/2}.
	$$
	
	Let
	$$
	\kappa := \min \{ \left((16d + 7L) \cdot (2 n_0)^8\right)^{-\alpha},
	\,  C_1 	^{-1}\} > 0,
	$$
	and note that $\kappa$ depends on
	$d,\alpha,\gamma, {\boldsymbol{\ell}}$ and ${\boldsymbol{c}}$.
	Moreover, we have $n \geq M^{8/(8\alpha + \gamma)}$
	which implies that
	\begin{align*}
		\kappa \, M^{-64\alpha/(8\alpha+\gamma)}
		&= \frac{\kappa}{M^8} \, M^{{8\gamma}/({8\alpha+\gamma})}
		\leq \frac{\kappa}{M^8} \cdot n^{\gamma}
		\\[1ex]
		&\leq  \frac{\kappa}{M^8}  \cdot C_1 
		\cdot
		C^{2^{{L}}-1}\cdot	
		n^{(2^{{L}}  -1)/2}
		\\[1ex]
		&\leq \frac 1 {M^8} \cdot C^{2^{{L}}-1}\cdot	
		n^{(2^{{L}}  -1)/2}.
	\end{align*}
	
	Note that the inclusion
	$c \Sigma^{\boldsymbol{\ell}, \boldsymbol{c},\varrho_2}_t
	\subset \Sigma^{\boldsymbol{\ell}, \boldsymbol{c},\varrho_2}_t$ holds true
	for any $c$ such that $c \in [-1,1]$.
	Consequently, in view of  \Cref{lem:MultidimensionalHatImplementation}
	and the fact that $5 \leq L \leq {\boldsymbol{\ell}}(n_0) \leq {\boldsymbol{\ell}}(n)$
	it follows that
	$$
	g_{M,y} = \kappa \, M^{-64\alpha/(8\alpha+\gamma)}
	\, \vartheta_{M,y} \in \Sigma^{\boldsymbol{\ell}, \boldsymbol{c},\varrho_2}_{16n^8d  + 7L}.
	$$
	
	By construction it is clear that
	$\| g_{M,y} \|_{L^\infty} \leq \| \vartheta_{M,y} \|_{L^\infty} \leq 1$.
	Furthermore, for $t \in \mathbb{N}$, we distinguish two cases:
	First, $t \geq 16n^8d  + 7L$, we have $g_{M,y} \in
	\Sigma^{\boldsymbol{\ell}, \boldsymbol{c},\varrho_2}_t$, hence
	$t^{\alpha} \, d_{\infty} (g_{M,y},
	\Sigma^{\boldsymbol{\ell}, \boldsymbol{c},\varrho_2}_t) = 0 \leq 1$.
	On the other hand, if $t \leq 16n^8d  + 7L$, in view of the fact that
	\[
	n \leq 1 + n_0 \, M^{8/(8\alpha+\gamma)}
	\leq 2 n_0 \, M^{8/(8\alpha+\gamma)},
	\]
	we have
	\begin{align*}
		t^{\alpha }\, d_\infty (g_{M,y},
		\Sigma^{\boldsymbol{\ell}, \boldsymbol{c},\varrho_2}_t)
		& \leq \bigl(16n^8d + 7L\bigr)^{\alpha}
		\, \| g_{M,y} \|_{L^\infty}
		\leq \bigl(16d + 7L\bigr)^{\alpha} \, n^{8\alpha}
		\, \| g_{M,y} \|_{L^\infty}
		\\
		&\leq
		\bigl(16d + 7L\bigr)^{\alpha} \,n^{8\alpha} \,
		\kappa \, M^{-64\alpha / (8\alpha+\gamma)}
		\\
		& \leq \bigl(16d + 7L\bigr)^{\alpha} \,
		(2 n_0)^{8\alpha} \,
		M^{64\alpha/(8\alpha+\gamma)}\,
		\kappa  \,
		M^{-64\alpha / (8\alpha+\gamma)}
		\\[1ex]
		&\leq 1.
	\end{align*}
	Hence, $\Gamma_{\alpha, \infty}^{\varrho_2}(g_{M,y}) \leq 1$,
	consequently \Cref{lem:ApproximationSpaceProperties}
	shows that
	$\| g_{M,y} \|_{A_{\boldsymbol{\ell},
			\boldsymbol{c},\varrho_2}^{\alpha,\infty}} 	
	\leq 1$.
\end{proof}

\section{Error bounds for uniform approximation}
\label{sec:UniformApproximationErrorBounds}

Theorem \ref{thm:ErrorBoundUniformApproximation}
is the main theorem in this section,
where we characterize the error bound
based on samples complexity
for the computational problem
of uniform approximation on the ReQU neural network
approximation space $A_{\boldsymbol{\ell},
	\boldsymbol{c},\varrho_2}^{\alpha,\infty}([0,1]^d)$.
We need \Cref{lem:NetworkLipschitzEstimate} in our main theorem,
in this lemma we derive an upper bound on  Lipschitz constant
of functions $F \in \Sigma^{\boldsymbol{\ell}, \boldsymbol{c},\varrho_2}_n$
on bounded domain and ReQU
activation function.
For any matrix $A \in \mathbb{R}^{k \times m}$,
define $\|A\|_{\infty}:=\max _{i, j}\left|A_{i, j}\right|$
and denote by $\|A\|_{\ell^{0}}$ the number of non-zero entries of $A$.

\begin{lemma}\label{lem:NetworkLipschitzEstimate}
	Let $B(0, R)$ be the ball of center $0$ and radius $R\geq 1$ in $\mathbb{R}^d$,
	${\boldsymbol{\ell}} : \mathbb{N} \to \mathbb{N} \cup \{ \infty \}$
	and ${\boldsymbol{c}} : \mathbb{N} \to [1,\infty]$.
	Let $n \in \mathbb{N}$ and assume that $L := {\boldsymbol{\ell}}(n)$ and
	$C := {\boldsymbol{c}}(n)$ are finite.
	Then each
	$F \in \Sigma^{\boldsymbol{\ell}, \boldsymbol{c},\varrho_2}_n$
	satisfies:
	\begin{align*}
		\operatorname{Lip}_{(B(0, R), \| \cdot \|_{\ell^1}) \to \mathbb{R}}(F)
		&\leq
		2^{2^{{L}}+{L}-3}
		R^{2^{{L}-1} -1}\cdot
		C^{2^{{L}}-1}\cdot	
		n^{(2^{{L}}  -1)/2}
		\intertext{and}
		\operatorname{Lip}_{(B(0, R), \| \cdot \|_{\ell^\infty}) \to \mathbb{R}}(F)
		&\leq d \cdot
		2^{2^{{L}}+{L}-3}
		R^{2^{{L}-1} -1}\cdot
		C^{2^{{L}}-1}\cdot	
		n^{(2^{{L}}  -1)/2}.
	\end{align*}
	Furthermore, we have
	\begin{align*}
		\operatorname{Lip}_{([0,1]^d, \| \cdot \|_{\ell^1}) \to \mathbb{R}}(F)
		&\leq
		2^{2^{{L}}+{L}-3}
		d^{(2^{{L}-1} -1)/2}\cdot
		C^{2^{{L}}-1}\cdot	
		n^{(2^{{L}}  -1)/2}
		\intertext{and}
		\operatorname{Lip}_{([0,1]^d, \| \cdot \|_{\ell^\infty}) \to \mathbb{R}}(F)
		&\leq d \cdot
		2^{2^{{L}}+{L}-3}
		d^{(2^{{L}-1} -1)/2}\cdot
		C^{2^{{L}}-1}\cdot	
		n^{(2^{{L}}  -1)/2}.
	\end{align*}
\end{lemma}

\begin{proof}
	Let $F$ be an element of
	$\Sigma^{\boldsymbol{\ell}, \boldsymbol{c},\varrho_2}_n$,
	$\Phi$ is a network given by
	$\Phi = \big( (A_1,b_1),\dots,(A_{\widetilde{L}},b_{\widetilde{L}}) \big)$
	such that $F = R_{\varrho_2} \Phi$ 	
	where $\widetilde{L} \leq {\boldsymbol{\ell}}(n) = L$ and
	$\max\{\| A_j \|_{\infty}, \| b_j \|_{\infty}\} \leq \| \Phi \|_{\mathcal{NN}} \leq {\boldsymbol{c}}(n) = C$,
	as well as $(\| A_j \|_{\ell^0} +\| b_j \|_{\ell^0} )\leq W(\Phi) \leq n$ for all
	$j \in \underline{\widetilde{L}}$.
	Let $T_j$ be an affine function defined as 
	$T_j( x) := A_j \, x + b_j$,
	where $A_j \in \mathbb{R}^{N_j \times N_{j-1}}$,
	$b_j \in \mathbb{R}^{N_j}$  and set	
	\[
	p_j
	= 
	\begin{cases}
		1, & \text{if } j \text{ is even}, \\
		\infty ,                        & \text{if } j \text{ is odd}.
	\end{cases}
	\]	
	We recall that $\|A\|_{\infty}:=\max _{i, j}\left|A_{i, j}\right|$
	and that $\|A\|_{\ell^{0}}$ is the number of non-zero entries of $A$,
	for any matrix $A \in \mathbb{R}^{k \times m}$.
	In view of the Step 1 in the proof of \cite[Lemma 4.1]{grohs2021},
	the function
	\(
	T_j :
	\bigl(
	\mathbb{R}^{N_{j-1}}, \| \cdot \|_{\ell^{p_{j - 1}}}
	\bigr) \to
	\bigl(
	\mathbb{R}^{N_j}, \| \cdot \|_{\ell^{p_j}}
	\bigr)
	\)
	is Lipschitz with
	\[
	\operatorname{Lip}(T_j)
	= \| A_j \|_{\ell^{p_{j-1}} \to \ell^{p_j}}
	\leq \begin{cases}
		\| A_j \|_{\infty} \, \| A_j \|_{\ell^0} \leq C n,
		& \text{if } j \text{ is even}, \\
		\| A_j \|_{\infty} \leq C , 
		& \text{if } j \text{ is odd},
	\end{cases}
	\]
	we refer the reader to \cite[Lemma 4.1]{grohs2021}
	for more details concerning this step.

	The major difference in our constructed network is that our activation function
	is not globally Lipschitz, instead it is locally Lipschitz.
	Thus, it is impossible to get the inequalities in
	Lemma \ref{lem:NetworkLipschitzEstimate} for unbounded domains
	of $\mathbb{R}^d$, which is
	the main reason behind the restriction of our result to bounded domains.
	
	Since for any bounded set $K$ such that $K\subset \mathbb{R}^{N_{j-1}}$,
	we have $\varrho_2(K)$ and $T_j(K)$
	are bounded subsets of $\mathbb{R}^{N_{j-1}}$,
	for any $j$.
	More precisely we have
	\[
	T_j(B(0, R)) \subset B(0, 2\sqrt{n}CR),
	\]
	indeed, for any fixed $x\in B(0,R)\subset \mathbb{R}^{N_{j-1}}$,
	$A\in \mathbb{R}^{N_{j}\times N_{j-1}}$ and	$b \in \mathbb{R}^{N_{j}}$, we have
\begin{multline}\label{eq:T_j_bound}
	\| T_j(x)\|_{\ell^2}^2 \!
	\leq \!\sum_{i=1}^{N_j}
	\left|\sum_{k=1}^{N_{j-1}} (A_j)_{i,k}x_k + b_i\right|^2\!\!
	\!\leq \!
	\sum_{i=1}^{N_{j}} \left[\left|\sum_{k=1}^{N_{j-1}}(A_j)_{i,k}x_k\right|^2
	\!\!+\! 2\left|\sum_{k}^{N_{j-1}}(A_j)_{i,k}x_k\right||b_i|
	\!+ \!|b_i|^2\right]
	\\[1ex] 
	\leq
	\sum_{i=1}^{N_{j}}\left[ \left(\sum_{k=1}^{N_{j-1}}|(A_j)_{i,k}|^2\sum_{k=1}^{N_{j-1}}|x_k|^2\right)
	+ 2\left(\sum_{k=1}^{N_{j-1}}|(A_j)_{i,k}|^2\right)^{1/2}
	\left(\sum_{k=1}^{N_{j-1}}|x_{k}|^2\right)^{1/2}|b_i|
	+ |b_i|^2\right]
	\\[1ex] 
	\leq
	\sum_{i=1}^{N_{j}}\left[ \left(\sum_{k=1}^{N_{j-1}}|(A_j)_{i,k}|^2R^2\right)
	+ 2\left(\sum_{k=1}^{N_{j-1}}|(A_j)_{i,k}|^2\right)^{1/2}
	R|b_i|
	+ |b_i|^2\right]
	\\[1ex] 
	\leq	
	 \sum_{i=1}^{N_{j}}\!\!\left(\!\|A_j\|_{\infty}^2R^2
	 	\sum_{k=1}^{N_{j-1}}\mathds{1}_{(A_j)_{i,k}\neq 0}\!\right)
	\!\!+\! 2\|A_j\|_{\infty}R\sum_{i=1}^{N_{j}}\left(\sum_{k=1}^{N_{j-1}}\mathds{1}_{(A_j)_{i,k}\neq 0}\!\right)^{1/2}\!\!
	|b_i|
	+ \!\|b\|_{\infty}^2\sum_{i=1}^{N_{j}}|\mathds{1}_{b_i\neq 0}
	\\[1ex] 
	\leq
	\|A_j\|_{\infty}^2R^2\sum_{i=1}^{N_{j}}\sum_{k=1}^{N_{j-1}}\mathds{1}_{(A_j)_{i,k}\neq 0}
	+ 2\|A_j\|_{\infty}\|b\|_{\infty}R\sum_{i=1}^{N_{j}}\sum_{k=1}^{N_{j-1}}\mathds{1}_{(A_j)_{i,k}\neq 0}\mathds{1}_{b_i\neq 0}
	+ \|b\|_{\infty}^2\sum_{i=1}^{N_{j}}|\mathds{1}_{b_i\neq 0}
	\\[1ex] 
	\leq
	C^2R^2n+2C^2Rn+C^2n = nC^2(R^2+2R +1)\leq nC^2(R+1)^2
	\leq
	4nC^2R^2
\end{multline}

	\noindent where we used the fact that $ C, R\geq 1$.
	Furthermore, since $R\geq 1$ we have
	\begin{equation}\label{eq:varrho_2_bound}
		\varrho_2(B(0, R)) \subset B(0, R^2).
	\end{equation}
	For any $x, y\in B(0, R)$, we define $\{X_j\}_{j=1}^{\widetilde{L} - 1}$
	and $\{Y_j\}_{j=1}^{\widetilde{L} - 1}$,
	such that
	\begin{align*}
		X_1 &= T_1(x), &Y_1 &= T_1(y),
		\\
		X_2 &= T_2\circ \varrho_2 \circ T_1(x),
		&Y_2 &= T_2\circ \varrho_2 \circ T_1(y),
		\\
		&\vdots &\vdots&
		\\
		X_{\widetilde{L} - 1} &= T_{\widetilde{L} - 1}\circ \varrho_2\circ
		\dots \circ \varrho_2\circ T_1(x),
		&Y_{\widetilde{L} - 1} &= T_{\widetilde{L} - 1}\circ \varrho_2\circ
		\dots \circ \varrho_2\circ T_1(y).
	\end{align*}

	Using the fact that we can write the difference $\varrho_2(X_j)- \varrho_2(Y_j)$
	for any $X_j, Y_j\in B(0, R_j)$ as
	$$
	\varrho_2(X_j)- \varrho_2(Y_j) = 
	2(X_j-Y_j)\int_{0}^{1}\varrho_1(tX_j - (1-t)Y_j) dt,
	$$
	where $2\varrho_1$ is the derivative of $\varrho_2$.
	Furthermore, we have
	$$
	\left| \varrho_2(X_j)- \varrho_2(Y_j) \right|\leq 
	2\sup_{t\in [0,1]}\varrho_1(tX_j - (1-t)Y_j)|X_j-Y_j|
	\leq 4R_j|X_j-Y_j|,
	$$
	which implies that $\varrho_2$ is $4R_j$-Lipschitz function
	from $(B(0,R_j), \|\cdot \|_{\ell^q}) $ to $(B(0,R_j^2), \|\cdot \|_{\ell^q})$
	for any $q\in [1, \infty]$.
	Here we need to find out the values of $R_j$ such that
	$j\in \{1,\dots,{\widetilde{L}-1}\}$.
	In view of \eqref{eq:T_j_bound} and \eqref{eq:varrho_2_bound},
	we characterize $R_j$ by the sequence 
	$$
	R_0 = R,\; R_1 = 2\sqrt{n}CR_0,\; R_2 =  2\sqrt{n}CR_1^2,\;
	\dots,\; R_j = 2\sqrt{n}CR_{j-1}^2,
	$$
	hence we can be write $R_j$ as 
	$$
	R_j  =
	\begin{cases}
		2\sqrt{n}CR, &j=1
		\\[1ex]
		2\sqrt{n}CR_{j-1}^2, &j\geq 2.
	\end{cases}
	$$
	Next we give a characterization of $R_j$
	depending on the layer number $j$.
	Indeed, for any $j\in \{1, \dots,{\widetilde{L} - 1}\}$
	we have 
	\begin{equation}\label{eq:R_j_characterization}
		R_j = (2\sqrt{n}C)^{\sum _{k=1}^{j}2^{k-1}}R^{2^{j-1}}
		= (2\sqrt{n} C)^{2^j-1}R^{ 2^{j-1}}.
	\end{equation}
	Using \eqref{eq:R_j_characterization}, the result holds for
	$R_1 = (2\sqrt{n} C)^{{2^1-1}} R^{2^{0}} = 2\sqrt{n} CR$
	and 
	$R_2 = (2\sqrt{n} C)^{{2^2-1}} R^{2}
	=(2\sqrt{n} C)^3R^2$,
	assuming that $R_j$ is true, we have
	\begin{align*}
		R_{j+1} &= 2\sqrt{n} CR_j^2 =
		2\sqrt{n}C\left((2\sqrt{n}C)^{2^j-1}R^{ 2^{j-1}}\right)^2
		\\
		& =(2\sqrt{n} C)\cdot ((2\sqrt{n} C)^{2^j-1})^2(R^{ 2^{j-1}})^2
		=(2\sqrt{n} C)^{1+ 2^{j+1}-2} R^{2^{j}}
		\\
		& = (2\sqrt{n} C)^{2^{j+1}-1} R^{2^j},
	\end{align*}
	which concludes the induction argument.
	Consequently, our network 
	\[
	F
	= R_{\varrho_2} \Phi
	= T_{\widetilde{L}}
	\circ ({\varrho_2} \circ T_{\widetilde{L} - 1})
	\circ \cdots
	\circ ({\varrho_2} \circ T_1)
	:\quad (B(0,R), \| \cdot \|_{\ell^1}) \to (\mathbb{R},
	\| \cdot \|_{\ell^{p_{\widetilde{L}}}} ) = (\mathbb{R}, |\cdot|)
	\]
	is Lipschitz continuous on $B(0,R)$ as a composition of Lipschitz maps.
	Then, the overall Lipschitz constant for $R_{\varrho_2} \Phi$ is given by
	\begin{align*}
		\operatorname{Lip}(R_{\varrho_2} \Phi)
		&\leq C \cdot n_{\widetilde{L}}\prod_{j=1}^{\widetilde{L}-1}
		\bigl(4R_jC \cdot n_j\bigr)
		\leq
		\left(\prod_{j=1}^{\widetilde{L}-1} 
		4(2\sqrt{n}C)^{2^j-1}R^{ 2^{j-1}}\right)
		\cdot C^{\widetilde{L}} \cdot n^{\lfloor \widetilde{L} / 2 \rfloor}
		\\
		&= 
		4^{\widetilde{L}-1}
		(2\sqrt{n}C)^{\sum_{j=1}^{\widetilde{L}-1}(2^{j}-1)}
		R^{\sum_{j=1}^{\widetilde{L}-1}2^{j-1} }\cdot
		C^{\widetilde{L}}\cdot
		n^{\lfloor \widetilde{L} / 2 \rfloor}
		\\
		&=
		4^{\widetilde{L}-1}
		(2\sqrt{n}C)^{(2^{\widetilde{L}}-\widetilde{L}-1)}
		R^{2^{\widetilde{L}-1} -1}\cdot
		C^{\widetilde{L}}\cdot
		n^{\lfloor \widetilde{L} / 2 \rfloor}
		\\
		&=
		2^{2\widetilde{L}-2}2^{2^{\widetilde{L}}-\widetilde{L}-1}
		R^{2^{\widetilde{L}-1} -1}\cdot
		C^{2^{\widetilde{L}}-\widetilde{L}-1 +\widetilde{L}}\cdot
		n^{\frac 12(2^{\widetilde{L}}-\widetilde{L}-1 )+ \lfloor \widetilde{L} / 2 \rfloor}
		\\
		&\leq
		2^{2^{\widetilde{L}}+\widetilde{L}-3}
		R^{2^{\widetilde{L}-1} -1}\cdot
		C^{2^{\widetilde{L}}-1}\cdot	
		n^{(2^{\widetilde{L}}  -1)/2}
		\\
		&\leq
		2^{2^{{L}}+{L}-3}
		R^{2^{{L}-1} -1}\cdot
		C^{2^{{L}}-1}\cdot	
		n^{(2^{{L}}  -1)/2}
	\end{align*}
	where we used the notation $n_j := n$ if $j$
	is even and $n_j := 1$ otherwise and that $R, C \geq 1$.
	The final claim of the lemma holds, since
	$\| x \|_{\ell^1} \leq d \cdot \| x \|_{\ell^\infty}$
	for $x \in B(0, R)\subset \mathbb{R}^d$.
	
	\par
	
	The remaining part of the lemma is straightforward, since $[0,1]^d\subset B(0, \sqrt{d})$.
\end{proof}

\par

The bound in the previous lemma is the main reason for the choice of 
$\gamma^{\flat}$ and $\gamma^{\sharp}$ in \eqref{eq:GammaDefinition}.
Generalizing the upper bounds in the previous lemma
for different activation functions can be thought of as a future
line of research. Also one could investigate
the impact of the magnitude of the weights,
for instance if ${\boldsymbol{c}}(n)\leq\frac 1{\sqrt{n}}$ we remark that
the upper bound in the previous lemma depends only on the depth in the first case,
while in the second one it depends on the depth and the dimension of the data.

\par

\begin{theorem}\label{thm:ErrorBoundUniformApproximation}
	Let ${\boldsymbol{\ell}},{\boldsymbol{c}} : \mathbb{N} \to \mathbb{N} \cup \{ \infty \}$
	be non-decreasing, $d \in \mathbb{N}$ and $\alpha \in (0,\infty)$ be arbitrary
	and suppose that
	$\gamma^{\sharp}({\boldsymbol{\ell}},{\boldsymbol{c}}) < \infty$.
	Let $U_{\boldsymbol{\ell},
		\boldsymbol{c}, \varrho_2}^{\alpha,\infty}([0,1]^d)$ as in \eqref{eq:UnitBallDefinition}.
	Furthermore, let 
	$
	\iota_\infty : A_{\boldsymbol{\ell},
		\boldsymbol{c},\varrho_2}^{\alpha,\infty}([0,1]^d) \to C([0,1]^d).
	$
	Then, we have
	\[
	err^{\deterministic} \bigl(U_{\boldsymbol{\ell},
		\boldsymbol{c}, \varrho_2}^{\alpha,\infty}([0,1]^d), \iota_\infty\bigr),
  err^{\MonteCarlo}(U_{\boldsymbol{\ell},
		\boldsymbol{c}, \varrho_2}^{\alpha,\infty}([0,1]^d),\iota_\infty)
	\geq \alpha/d(\alpha+ \gamma^{\sharp}({\boldsymbol{\ell}},{\boldsymbol{c}}) ).
	\]
\end{theorem}

\begin{proof}
	Using \Cref{rem:GammaRemark} and the fact that $\gamma^{\sharp}({\boldsymbol{\ell}},{\boldsymbol{c}}) < \infty$,
	it follows that $L := {\boldsymbol{\ell}}^\ast < \infty$.
	Let $\gamma\in \mathbb{R}$, such that
	$1\leq \gamma^{\sharp}({\boldsymbol{\ell}},{\boldsymbol{c}})<\gamma $.
	In view of the definition of $\gamma^{\sharp}({\boldsymbol{\ell}},{\boldsymbol{c}})$
	in \eqref{eq:GammaDefinition},
	there exists
	a constant $C_0 = C_0(d, \gamma,{\boldsymbol{\ell}},{\boldsymbol{c}}) > 0$
	such that
	\begin{equation}\label{eq:C_0_bound}
		C^{2^{{L}}-1}\cdot	
		n^{(2^{{L}}  -1)/2}
		\leq C_0 \cdot n^{\gamma}
	\end{equation}
	for all $n \in \mathbb{N}$.
	Suppose that 
	\begin{equation}\label{eq:N_n_construction}
		N := \big\lfloor m^{{1/d}} \big\rfloor 
		\qquad \text{ and } \qquad
		n := \big\lceil m^{1/d( \alpha+\gamma)} \big\rceil,
	\end{equation}
	for any  $m \in \mathbb{N}$.
	In view of the previous construction, we get that
	$N\geq 1$ and $n \in \mathbb{N}$.
	Furthermore, let
	$$
	I := \bigl\{ 0, \frac 1{N}, \dots,
	\frac{N-1} {N} \bigr\}^d,
	$$
	it follows that $I \subset [0,1]^d$.
	Let $C := {\boldsymbol{c}}(n)$ and
	$$
	\mu := d \cdot 	2^{2^{{L}}+{L}-3}
	d^{(2^{{L}-1} -1)/2}\cdot
	C^{2^{{L}}-1}\cdot	
	n^{(2^{{L}}  -1)/2}.
	$$
	Using \eqref{eq:C_0_bound}, we have
	$$
	\mu \leq d \cdot 2^{2^{{L}}+{L}-3}
	d^{(2^{{L}-1} -1)/2}\cdot C_0 \cdot n^{\gamma} =:C_1 \, n^{\gamma}
	$$
	such that $C_1$ depends on $\gamma,{\boldsymbol{\ell}},{\boldsymbol{c}}$ and $d$.
	Furthermore, 
	we have that
	$
	|I| = N^d \leq	m.
	$
	Recalling that
	\[
	U_{\boldsymbol{\ell},\boldsymbol{c}, \varrho_2}^{\alpha,\infty}([0,1]^d)
	=
	\bigl\{
	f \in A_{\boldsymbol{\ell},
		\boldsymbol{c},\varrho_2}^{\alpha,\infty}([0,1]^d)
	\colon
	\| f \|_{A_{\boldsymbol{\ell},
			\boldsymbol{c},\varrho_2}^{\alpha,\infty}([0,1]^d)}
	\leq 1
	\bigr\}.
	\]
	Let $\Lambda := \Psi(U_{\boldsymbol{\ell},
		\boldsymbol{c}, \varrho_2}^{\alpha,\infty}([0,1]^d))$ such that
	$\Psi$ defined as follows:
	\[
	\Psi : \quad
	C([0,1]^d) \to \mathbb{R}^I, \quad
	f \mapsto \big( f(i) \big)_{i \in I} .
	\]
	For any element $z\in \Lambda$, $z$ can be written as $z=(z_i)_{i\in I}$.
	Moreover, there exits $f_z \in U_{\boldsymbol{\ell},
		\boldsymbol{c}, \varrho_2}^{\alpha,\infty}([0,1]^d)$
	such that $z = \Psi(f_z)$.
	Furthermore, by construction of $U_{\boldsymbol{\ell},
		\boldsymbol{c}, \varrho_2}^{\alpha,\infty}([0,1]^d)$
	and the fact that  $f_z \in U_{\boldsymbol{\ell},
		\boldsymbol{c}, \varrho_2}^{\alpha,\infty}([0,1]^d)$,
	we have	$\| f_z \|_{A_{\boldsymbol{\ell},
			\boldsymbol{c},\varrho_2}^{\alpha,\infty}([0,1]^d)}	\leq 1$,
	which implies that $\Gamma_{\alpha, \infty}^{\varrho_2}(f_z) \leq 1$ ( in view of 
	\Cref{lem:ApproximationSpaceProperties}).
	Consequently, by definition of $\Gamma_{\alpha, \infty}^{\varrho_2}(f_z)$
	and the fact that $\Gamma_{\alpha, \infty}^{\varrho_2}(f_z) \leq 1$,
	there exists
	$F_z \in \Sigma^{\boldsymbol{\ell}, \boldsymbol{c},\varrho_2}_n$
	satisfying
	$$
	\| f_z - F_z \|_{L^\infty} \leq 2 \cdot n^{-\alpha}.
	$$
	Given this choice, let
	\[
	\Xi : \quad
	\mathbb{R}^I \to C([0,1]^d), \quad
	z \mapsto \begin{cases}
		F_z, & \text{if } z \in \Lambda, \\
		0 , & \text{otherwise} .
	\end{cases}
	\]
	Let $f \in U_{\boldsymbol{\ell},
		\boldsymbol{c}, \varrho_2}^{\alpha,\infty}([0,1]^d)$ and
	$z := \Psi(f) \in \Lambda$.
	Similar to the argument above, there exists
	$F \in \Sigma^{\boldsymbol{\ell}, \boldsymbol{c},\varrho_2}_n$
	such that
	$\| f - F \|_{L^\infty} \leq2 \cdot n^{-\alpha}$.
	Consequently, we have
	$$
	f(i) = (\Psi(f))_i = z_i = (\Psi(f_z))_i = f_z(i),
	\quad \text{for any}\; i \in I
	$$
	which implies that
	\[
	| F(i) - F_z (i) |
	\leq |F(i) - f(i)| + |f_z (i) - F_z (i)|
	\leq \| F - f \|_{L^\infty} + \| f_z - F_z \|_{L^\infty}
	\leq 4 \cdot n^{-\alpha} .
	\]
	In view of the fact that $F$ and $F_z$ belong to
	$\Sigma^{\boldsymbol{\ell}, \boldsymbol{c},\varrho_2}_n$,
	\Cref{lem:NetworkLipschitzEstimate} implies that
	$F$ and $F_z$ are
	$(d \cdot	2^{2^{{L}}+{L}-3}
	d^{(2^{{L}-1} -1)/2}\cdot
	C^{2^{{L}}-1}\cdot	
	n^{(2^{{L}}  -1)/2})$
	-Lipschitz from
	$([0,1]^d, \|\cdot \|_{\ell^\infty})$ to $(\mathbb{R}, |\cdot|)$.
	
	Then $F - F_z : ([0,1]^d, \| \cdot \|_{\ell^\infty}) \to (\mathbb{R}, |\cdot|)$
	is Lipschitz continuous	with Lipschitz constant at most $2 \mu$.
	Furthermore, for any $x \in [0,1]^d$,
	we can choose $i = i(x) \in I$ such that
	$\| x - i \|_{\ell^\infty} \leq N^{-1}$.
	Hence, we have 
	\begin{align*}
		|(F - F_z)(x)|
		&\leq
		|(F - F_z)(x)-(F - F_z)(i)|+|(F - F_z)(i)|
		\\&
		\leq \frac{2 \mu}{ N} + |(F - F_z)(i)|
		\leq \frac{2 \mu}{ N} + 4 \, n^{-\alpha} .
	\end{align*}
	Consequently, we have
	\begin{equation}\label{eq:upper_bound_F-Fz}
		\| F - F_z \|_{L^\infty}
		\leq \frac{2 \mu}{ N} + 4 \, n^{-\alpha}.
	\end{equation}
	Since $\Xi(\Psi(f)) = \Xi(z) = F_z$, \Cref{eq:upper_bound_F-Fz} implies that
	\[
	\big\| f - \Xi(\Psi(f)) \big\|_{L^\infty}
	\leq \| f - F \|_{L^\infty} + \| F - F_z \|_{L^\infty}
	\leq 6 \, n^{-\alpha} + \frac{2 \mu}{N}.
	\]
	
	It remains to note that our choice of
	$N$ in \eqref{eq:N_n_construction}
	implies	that $m^{{1/d}} \leq 1 + N \leq 2 N$
	and that $\frac{1}{N}\leq 2m^{-1/d}$.
	Moreover, in view of the construction given in \eqref{eq:N_n_construction},
	we have
	$$
	n \leq 1 + m^{1/d(\gamma + \alpha)}
	\leq 2 \,m^{1/d(\gamma + \alpha)}.
	$$
	Hence, recalling that $\mu \leq C_1 \, n^{\gamma}$, it follows that
	\begin{align*}
		\frac{\mu}{N}
		&\leq 2m^{-1/d}  \cdot \mu
		\leq 2m^{-1/d}  \cdot  C_1 n^{\gamma}
		\\&
		\leq  2m^{-1/d}  \cdot
		C_1 (2 \,m^{1/d(\gamma + \alpha)} )^{\gamma}
		= 2^{\gamma+1} C_1
		m^{\gamma/d(\gamma + \alpha)-\frac{1}{d}}
		\\&
		=2^{\gamma+1} C_1
		m^{-\alpha/d(\gamma + \alpha)}.	
	\end{align*}
	
	\par	
	
	Observing that in view of \eqref{eq:N_n_construction}, we have
	$n \geq m^{1/d(\gamma + \alpha)}$, which implies that
	$$
	n^{-\alpha} \leq m^{-\alpha/d(\gamma + \alpha)}.
	$$
	Finally, using the previous approximations, we get the following
	upper bound
	$$
	\| f - \Xi(\Psi(f)) \|_{L^\infty}
	\leq C_2 \cdot m^{-\alpha/d(\gamma + \alpha)}
	$$
	where $C_2$ is a positive constant depends on
	$d,\gamma,{\boldsymbol{\ell}}$ and ${\boldsymbol{c}}$.
	It follows that
	\[
	err^{\deterministic} ( U_{\boldsymbol{\ell},
		\boldsymbol{c}, \varrho_2}^{\alpha,\infty}([0,1]^d), \iota_\infty)
	\geq \alpha/d(\gamma + \alpha).
	\]
    In order to conclude our proof we use 
    \cite[Proposition 3.3]{Heinrich1994}
	which then implies that
    \[
	err^{\MonteCarlo} ( U_{\boldsymbol{\ell},
		\boldsymbol{c}, \varrho_2}^{\alpha,\infty}([0,1]^d), \iota_\infty)
	\geq \alpha/d(\gamma + \alpha).
	\]
    hence the claim of the theorem holds.
\end{proof}

\section{Hardness of uniform approximation}
\label{sec:UniformApproximationHardness}
	
The following result is our second main result of the current paper.
In \Cref{thm:UniformApproximationHardness} we determine an upper bound
for the best possible convergence rate
with respect to the number of samples.

\begin{theorem}\label{thm:UniformApproximationHardness}
	Let ${\boldsymbol{\ell}}: \mathbb{N} \to \mathbb{N}_{\geq 5} \cup \{ \infty \}$
	and ${\boldsymbol{c}} : \mathbb{N} \to \mathbb{N} \cup \{ \infty \}$ be non-decreasing functions.
	Given $d \in \mathbb{N}$ and $\alpha \in (0,\infty)$,
	let $U_{\boldsymbol{\ell},
		\boldsymbol{c}, \varrho_2}^{\alpha,\infty}([0,1]^d)$
	be defined as in \Cref{eq:UnitBallDefinition},
	$\gamma^{\flat}({\boldsymbol{\ell}},{\boldsymbol{c}})$
	from \eqref{eq:GammaDefinition}
	and consider the embedding
	$\iota_\infty :
	A_{\boldsymbol{\ell},
		\boldsymbol{c},\varrho_2}^{\alpha,\infty}([0,1]^d)
	\hookrightarrow C([0,1]^d)$.
	Then
	\[
	err^{\deterministic}(U_{\boldsymbol{\ell},
		\boldsymbol{c}, \varrho_2}^{\alpha,\infty}([0,1]^d),\iota_\infty), 	err^{\MonteCarlo}(U_{\boldsymbol{\ell},
		\boldsymbol{c}, \varrho_2}^{\alpha,\infty}([0,1]^d),\iota_\infty)
	\leq
	{{64\alpha}/d
		({8\alpha+\gamma^{\flat}({\boldsymbol{\ell}},{\boldsymbol{c}})})}.
	\]
\end{theorem}

\begin{proof}
	Let $0 < \gamma < \gamma^{\flat}({\boldsymbol{\ell}},{\boldsymbol{c}})$
	and choose an arbitrary $m \in \mathbb{N}$ and let
	$M := 4 \lceil m^{1/d}\rceil$.	
	Furthermore, let $I_m := \underline{2 \lceil m^{1/d}\rceil}^d \times \{ \pm 1 \}$ and 
	$$
	y^{(\boldsymbol{i})}
	:= \frac {2\left(\boldsymbol{i} - (1,\dots,1)\right)+(1,\dots,1)}{M} \in \mathbb{R}^d
	\quad \text{ for any }\boldsymbol{i}
	\in \underline{2 \lceil m^{1/d}\rceil}^d.
	$$
	In view of the previous construction it follows that
	\begin{align*}
		y^{(\boldsymbol{i})} + (-\frac 1M, \frac 1M)^d
		& = \frac {2\left(\boldsymbol{i} - (1,\dots,1)\right)+(1,\dots,1)}{M}
		+  (-\frac 1M, \frac 1M)^d
		\\
		& = \frac 1M \left({2\left(\boldsymbol{i} - (1,\dots,1)\right)+(1,\dots,1)}
		+ (-1, 1)^d \right)
		\\
		&=\frac{1}{M} \left(2\left(\boldsymbol{i} - (1,\dots,1)\right)
		+ (0, 2)^d \right)
		\\
		&=\frac{2}{M} \Big(\boldsymbol{i} - (1,\dots,1) + (0,1)^d \Big)
		\subset (0,1)^d.
	\end{align*}
	Using $\vartheta_{M,y^{(\boldsymbol{i})}}$ defined in \Cref{lem:MultidimensionalHatProperties},
	the fact that $y^{(\boldsymbol{i})} + (-\frac 1M, \frac 1M)^d\subset (0,1)^d$
	implies that $\vartheta_{M,y^{(\boldsymbol{i})}}$
	have disjoint supports contained in $[0,1]^d$, for any
	$\boldsymbol{i} \in \underline{2 \lceil m^{1/d}\rceil}^d$.
	In view of \Cref{lem:MultidimensionalHatInSpace}
	there exists a constant
	$\kappa_1>0$ depends on $\gamma,\alpha,d,{\boldsymbol{\ell}}$
	and ${\boldsymbol{c}}$ such that for any $(\boldsymbol{i},\nu) \in I_m$
	the following holds true:
	\[
	f_{\boldsymbol{i},\nu}
	:= \kappa_1
	\cdot M^{-64\alpha/(8\alpha+\gamma)}
	\cdot \nu
	\cdot \vartheta_{M,y^{(\boldsymbol{i})}}
	\in U_{\boldsymbol{\ell}, \boldsymbol{c}, \varrho_2}^{\alpha,\infty}([0,1]^d).
	\]
	Let $A \in \Alg_m ( U_{\boldsymbol{\ell}, \boldsymbol{c}, \varrho_2}^{\alpha,\infty}([0,1]^d), C([0,1]^d))$
	then, by definition, there exist $m$ sample points in $[0,1]^d$,
	i.e.,
	$X = (x^{(1)},\dots,x^{(m)}) \in\left( [0,1]^d\right)^m$
	and a function $Q : \mathbb{R}^m \to \mathbb{R}$
	with the property that
	$$
	A(f) = Q(f(x_1),\dots,f(x_m))\text{ for all }
	f \in U_{\boldsymbol{\ell}, \boldsymbol{c}, \varrho_2}^{\alpha,\infty}([0,1]^d).
	$$
	Furthermore, let
	\(\Gamma_{X}
	:= \big\{
	\boldsymbol{i} \in \underline{2 \lceil m^{1/d}\rceil}^d
	\colon
	\forall \, n \in \underline{m} : \vartheta_{M,y^{(\boldsymbol{i})}} (x_n) = 0
	\big\}
	\)
	and denote the complement of the set $\Gamma_X$ in $\underline{2 \lceil m^{1/d}\rceil}^d$
	by
	$$
	\mathcal{C}(\Gamma_X , \underline{2 \lceil m^{1/d}\rceil}^d)=
	\underline{2 \lceil m^{1/d}\rceil}^d \setminus \Gamma_X.
	$$
	The definition of the previous set implies that
	for each $\boldsymbol{i} \in
	\mathcal{C}(\Gamma_X , \underline{2 \lceil m^{1/d}\rceil}^d)$,
	there exists $n_{\boldsymbol{i}} \in \underline{m}$
	such that $\vartheta_{M,y^{(\boldsymbol{i})}}(x_{n_{\boldsymbol{i}}}) \neq 0$.
	Consequently, $	\mathcal{C}(\Gamma_X , \underline{2 \lceil m^{1/d}\rceil}^d) \to \underline{m},
	\boldsymbol{i} \mapsto n_{\boldsymbol{i}}$
	is an injective map,
	since $\vartheta_{M,y^{(\boldsymbol{i})}}\vartheta_{M,y^{(\boldsymbol{j})}} = 0$
	for $\boldsymbol{i}, \boldsymbol{j} \in  \underline{2 \lceil m^{1/d}\rceil}^d$
	with $\boldsymbol{i}\neq \boldsymbol{j}$.
	Hence $|\mathcal{C}(\Gamma_X , \underline{2 \lceil m^{1/d}\rceil}^d)| \leq m$ 
	which implies that $|\Gamma_X| \geq ({2 \lceil m^{1/d}\rceil})^d - m \geq m$. 
	Then for any $\boldsymbol{i} \in \Gamma_X$
	and $\nu \in \{ \pm 1 \}$, we have $f_{\boldsymbol{i},\nu}(x_n) = 0$
	for all $n \in \underline{m}$
	and  $A(f_{\boldsymbol{i},\nu}) = Q(0,\dots,0)$.
	Therefore for any  $\boldsymbol{i} \in \Gamma_X$
	we have
	\begin{equation}
		\begin{split}
			\avsum_{\nu \in \{ \pm 1 \}}
			\| f_{\boldsymbol{i},\nu} - A(f_{\boldsymbol{i},\nu}) \|_{L^\infty}
			&=\frac 12 \left( \| f_{\boldsymbol{i},1} - A(f_{\boldsymbol{i},1}) \|_{L^\infty}
			+ \| f_{\boldsymbol{i},-1} - A(f_{\boldsymbol{i},-1}) \|_{L^\infty}\right)
			\\
			& =\frac 12 \left(\| f_{\boldsymbol{i},1} - Q(0,\dots,0) \|_{L^\infty}
			+\| - f_{\boldsymbol{i},1} - Q(0,\dots,0) \|_{L^\infty}\right)
			\\[1ex]
			&= \frac 12 \left(\| f_{\boldsymbol{i},1} - Q(0,\dots,0) \|_{L^\infty}
			+ \| Q(0,\dots,0) + f_{\boldsymbol{i},1} \|_{L^\infty} \right)
			\\[1ex]
			& \geq\frac 12\left( \| f_{\boldsymbol{i},1} - Q(0,\dots,0)
			+ Q(0,\dots,0) + f_{\boldsymbol{i},1} \|_{L^\infty}\right)
			\\[1ex]
			&=    \| f_{\boldsymbol{i},1} \|_{L^\infty}
			=   \kappa_1 \cdot M^{-64\alpha/(8\alpha+\gamma)}.
		\end{split}
		\label{eq:UniformApproximationHardnessElementaryEstimate}
	\end{equation}
	Recalling that $M= 4{\lceil m^{1/d}\rceil}$ 
	which implies that $M \leq 1+  4{ m^{1/d}}\leq 8m^{1/d}$,
	and hence  
	$$
	M^{-{64\alpha}/({8\alpha+\gamma})}
	\geq 8^{-{64\alpha}/
		({8\alpha+\gamma})}
	m^{-{{64\alpha}}/{d(8\alpha+\gamma)}}
	\geq \frac 1{8^8}
	m^{-{{64\alpha}}/{d(8\alpha+\gamma)}}.
	$$
	The fact that $|\Gamma_X| \geq m$,
	$\lceil m^{1/d}\rceil \leq 1+m^{1/d}\leq 2m^{1/d}$ and
	\Cref{eq:UniformApproximationHardnessElementaryEstimate}
	imply that
	\begin{align*}
		\avsum_{(\boldsymbol{i},\nu) \in I_m}
		\| f_{\boldsymbol{i},\nu} - A(f_{\boldsymbol{i},\nu}) \|_{L^\infty}
		& \geq \frac1{(2\lceil m^{1/d}\rceil )^d} \,
		\sum_{\boldsymbol{i} \in \Gamma_X} 
		\avsum_{\nu \in \{ \pm 1 \}}
		\| f_{\boldsymbol{i},\nu} - A(f_{\boldsymbol{i},\nu}) \|_{L^\infty}
		\\[1ex]
		& \geq \frac1{(2\lceil m^{1/d}\rceil )^d} 
		\cdot |\Gamma_X|
		\cdot \kappa_1
		\cdot M^{-64\alpha/(8\alpha+\gamma)}
		\\[1ex]
		&\geq \frac{1}{4^d}\cdot \kappa_1
		\cdot  M^{-64\alpha/(8\alpha+\gamma)}
		\\[1ex]
		& \geq \frac{1}{2^{2d}}\cdot \kappa_1\cdot \frac 1{2^{24}}
		m^{-{64\alpha}/d({8\alpha+\gamma})}
		= \kappa \cdot m^{-{64 \alpha}/d({8\alpha+\gamma})},
	\end{align*}
	where $\kappa := \frac{1}{2^{2d+24}}\cdot \kappa_1$.
	Consequently, we show that there is a constant
	$\kappa = \kappa(d,\alpha,\gamma,{\boldsymbol{\ell}},{\boldsymbol{c}}) > 0$
	and a family of functions
	$(f_{\boldsymbol{i},\nu})_{(\boldsymbol{i},\nu) \in I_m}
	\subset U_{\boldsymbol{\ell}, \boldsymbol{c}, \varrho_2}^{\alpha,\infty}([0,1]^d)$
	such that  for any
	\(
	A \in \Alg_m\bigl(U_{\boldsymbol{\ell},
		\boldsymbol{c}, \varrho_2}^{\alpha,\infty}([0,1]^d),C([0,1]^d)\bigr)
	\)
	it follows that
	\begin{equation}
		\avsum_{(\boldsymbol{i},\nu) \in I_m}
		\big\| f_{\boldsymbol{i},\nu} - A(f_{\boldsymbol{i},\nu}) \big\|_{L^\infty}
		\geq  \kappa \cdot m^{-{64\alpha}/d({8\alpha+\gamma})}.
		\label{eq:AverageCaseHardnessUniformApproximation}
	\end{equation}
	In view of 
	\Cref{lem:MonteCarloHardnessThroughAverageCase} it follows that
	\[
	err^{\deterministic}(U_{\boldsymbol{\ell},
		\boldsymbol{c}, \varrho_2}^{\alpha,\infty}([0,1]^d),\iota_\infty), err^{\MonteCarlo}(U_{\boldsymbol{\ell},
		\boldsymbol{c}, \varrho_2}^{\alpha,\infty}([0,1]^d),\iota_\infty)
	\leq {{64\alpha}/d({8\alpha+\gamma})},
	\]
	which implies the claim of the theorem.
\end{proof}


\appendix
\section{Appendix: Proof of Lemma \ref{lem:ApproximationSpaceProperties}}\label{appendix}

\begin{lemma}\label{lem:NetworkSetsClosedUnderSummation}
	Let $\widetilde{\boldsymbol{\ell}}(n) := \min \{ \boldsymbol{\ell}(n), n \}$,
	for  any $n \in \mathbb{N}_{\geq 2}$,
	it follows that
\begin{align}
	\Sigma^{\boldsymbol{\ell}, \boldsymbol{c}, \varrho_2}_n& =
	\Sigma^{\tilde{\boldsymbol{\ell}}, \boldsymbol{c}, \varrho_2}_n
	\label{eq:sigma_equality}
	\intertext{ furthermore, if the realization is bounded in $[-1,1]$,
		we get}
	\Sigma^{\boldsymbol{\ell}, \boldsymbol{c}, \varrho_2}_n
	+\Sigma^{\boldsymbol{\ell}, \boldsymbol{c}, \varrho_2}_n
	&\subset
	\Sigma^{\boldsymbol{\ell}, \boldsymbol{c}, \varrho_2}_{9n},
	\label{eq:sigma_imbedding}
\end{align}
\end{lemma}

\begin{proof}
	We start by showing that 
	\(
	\Sigma^{\boldsymbol{\ell}, \boldsymbol{c}, \varrho_2}_n =
	\Sigma^{\tilde{\boldsymbol{\ell}}, \boldsymbol{c}, \varrho_2}_n
	\).
	If $\widetilde{\boldsymbol{\ell}}(n ) = \boldsymbol{\ell}(n ) $
	then the result is trivial.
	Hence, let $n \in \mathbb{N}$ be fixed, 
	then trough an induction argument on $\ell \in \mathbb{N}_{\geq n}$ 
	in order to show that
	$\Sigma^{\boldsymbol{\ell}, \boldsymbol{c}, \varrho_2}_n \subset 
	\Sigma^{n, \boldsymbol{c}, \varrho_2}_n$.
	The case where $\ell = n$ is straightforward
	therefore it implies the first step in the induction.
	Hence, we assume that 
	$\Sigma^{{\ell}, \boldsymbol{c}, \varrho_2}_n
		\subset
			\Sigma^{n, \boldsymbol{c}, \varrho_2}_n$
	for some $\ell \in \mathbb{N}_{\geq n}$.
	Let $\Phi$ be a network architecture such that 
	$\|\Phi\|_{\mathcal{NN}} \leq \boldsymbol{c}(n)$,
	$W(\Phi) \leq n$
	and $L(\Phi) \leq \ell + 1$,
	hence
	$R_{\varrho_2} \Phi \in \Sigma^{\ell+1, \boldsymbol{c}, \varrho_2}_n$.
	In view of the induction process,
	if $L(\Phi) \leq \ell$,
	it follows that $R_{\varrho_2} \Phi \in 
	\Sigma^{\ell, \boldsymbol{c}, \varrho_2}_n
	\subset 
	\Sigma^{n, \boldsymbol{c}, \varrho_2}_n$.
	Consequently, let $L(\Phi) = \ell+1$
	hence
	$\Phi = \big( (A_1,b_1), \dots, (A_{\ell+1}, b_{\ell+1}) \big)$
	where
	$b_j \in \mathbb{R}^{N_j}$ and $A_j \in \mathbb{R}^{N_j \times N_{j-1}}$,
	$j\in \{1, \dots, \ell+1\}$.
	Since 
	\( W(\Phi)
	=
	\sum_{j=1}^{\ell+1} \bigl(\| A_j \|_{\ell^0} + \| b_j \|_{\ell^0}\bigr)
		\leq n
	\)
	and that 
	$n+1
	\leq \ell+1 \leq \sum_{j=1}^{\ell+1} \bigl(\| A_j \|_{\ell^0} + \| b_j \|_{\ell^0}\bigr)$
	it follows that 
	$n+1\leq n$.
	Hence for some  $j_0\in \{1, \dots, \ell+1\}$,
	we have $A_{j_0} =0_{\mathbb{R}^{N_{j_0} \times N_{{j_0}-1}}}$ and
		$b_{j_0} = 0_{\mathbb{R}^{N_{j_0}}}$.
	If $j = \ell+1$, then
	\(
	R_{\varrho_2} \Phi  \equiv 0
		\in \Sigma^{n, \boldsymbol{c}, \varrho_2}_n.
	\)
	Now, we assume that ${j_0} \leq \ell$ and we let
	\[
	\Phi_{j_0}
	:= \big(
	(0_{N_{{j_0}+1} \times d}, b_{{j_0}+1}),
	(A_{{j_0}+2}, b_{{j_0}+2}),
	\dots,
	(A_{\ell+1}, b_{\ell+1})
	\big)
	.
	\]
	The fact that
	$\varrho_2(0) = 0$
	and that 
	\(
	(A_{j_0}, b_{j_0}) =
	(0_{\mathbb{R}^{N_{j_0} \times N_{{j_0}-1}}},
	0_{\mathbb{R}^{N_{j_0}}})
	\)
	imply that
	$R_{\varrho_2} {\Phi_{j_0}} = R_{\varrho_2} \Phi$.
	Furthermore,
	we have
	$\|{{\Phi_{j_0}}} \|_{\mathcal{NN}}
	\leq \|{\Phi}\|_{\mathcal{NN}} \leq \boldsymbol{c}(n)$,
	$W({\Phi_{j_0}}) \leq W(\Phi) \leq n$
	and
	$L({\Phi_{j_0}}) \leq \ell - j_0 + 1 \leq \ell$,
	this implies $R_{\varrho_2} \Phi
	\in \Sigma^{\ell, \boldsymbol{c}, \varrho_2}_n
	\subset \Sigma^{n, \boldsymbol{c}, \varrho_2}_n$,
	where the last inclusion holds by induction.
	Then, we get
	\(
	\Sigma^{\boldsymbol{\ell}, \boldsymbol{c}, \varrho_2}_n 
	\subset
	\Sigma^{\tilde{\boldsymbol{\ell}}, \boldsymbol{c}, \varrho_2}_n
	\).
	Now it remains to show that
	\(
	\Sigma^{\tilde{\boldsymbol{\ell}}, \boldsymbol{c}, \varrho_2}_n
	\subset
	\Sigma^{\boldsymbol{\ell}, \boldsymbol{c}, \varrho_2}_n 
	\).
	This step is trivial since if $R_{\varrho_2}\Phi
	\in \Sigma^{\tilde{\boldsymbol{\ell}}, \boldsymbol{c}, \varrho_2}_n$
	it implies that for the network $\Phi$
	we have
	$
	W(\Phi)\leq n$,
	$
	\|{\Phi}\|_{\mathcal{NN}} \leq \boldsymbol{c}(n)$
	and 
	$
	L(\Phi) \leq \tilde{\boldsymbol{\ell}}
	=
	\min \{ n, \boldsymbol{\ell}(n) \} \leq \boldsymbol{\ell}(n),
	$
	hence the result follows.
	All in all, we have
	\(
	\Sigma^{\boldsymbol{\ell}, \boldsymbol{c}, \varrho_2}_n =
	\Sigma^{\tilde{\boldsymbol{\ell}}, \boldsymbol{c}, \varrho_2}_n
	\).
	
	\medskip{}

	Now we are concerned with the second result of the lemma
	That is, we will show that 
	$\Sigma^{\boldsymbol{\ell}, \boldsymbol{c}, \varrho_2}_n
	+
	\Sigma^{\boldsymbol{\ell}, \boldsymbol{c}, \varrho_2}_n
	\subset 
	\Sigma^{\boldsymbol{\ell}, \boldsymbol{c}, \varrho_2}_{9n}$.
	For that aim, we let 
	$R_{\varrho_2} \Phi^1, R_{\varrho_2} \Phi^2	\in \Sigma^{\boldsymbol{\ell}, \boldsymbol{c}, \varrho_2}_n 
	=
	\Sigma^{\min\{n, \boldsymbol{\ell}\}, \boldsymbol{c}, \varrho_2}_n
	$
	(from \eqref{eq:sigma_equality})
	hence both networks $\Phi^1,\Phi^2$ satisfy
	\begin{align*}
		W(\Phi^1),\; W(\Phi^2) &\leq n
		\\
		\|{\Phi^1}\|_{\mathcal{NN}},\;
		\|{\Phi^2}\|_{\mathcal{NN}} &\leq \boldsymbol{c}(n)
		\intertext{and}
		L(\Phi^1), \;L(\Phi^2) &\leq \min \{ n, \boldsymbol{\ell}(n) \},
	\end{align*}
	Without loss of generality, we can assume that
	$k := L(\Phi^1)$ and $ \ell := L(\Phi^2) $ such that $k\leq \ell$.
	Let 
	\begin{align*}
		\Phi^1&= \big( (A_1^1,b_1^1),\dots,(A_k^1,b_k^1) \big)
		\intertext{and}
		\Phi^2&= \big( (A_1^2,b_1^2),\dots,(A_\ell^2,b_\ell^2) \big),
	\end{align*}
	such that
	$A_i^1\in \mathbb{R}^{N^1_{i}\times N^1_{i-1}}$
	$A_j^2\in \mathbb{R}^{N^2_{j}\times N^2_{j-1}}$
	and $b_i^1\in \mathbb{R}^{N^1_{i}}$
	and $b_j^2\in \mathbb{R}^{N^2_{j}}$,
	$i\in \{1, \dots, k\}$
	and 
	$j\in \{1,\dots, \ell\}$
	where $N_0^1 = N_0^2 = d$
	and $N_k^1 = N_\ell^2 =1$.	
	
	In order to sum up the two realizations $R_{\varrho_2} \Phi^1$
	and $R_{\varrho_2} \Phi^2$ first we need to parallelize them.
	Hence they need to have the same number of layers.
	Consequently, we define ${\Phi^{e}} := \Phi^1$
	if $k = \ell$.
	If otherwise $k < \ell$,
	we let 
	$$
		\Gamma:=
			\big(\tfrac 14\left(
			\begin{smallmatrix}
				1 &-1
				\\
				-1& 1
			\end{smallmatrix}
			\right),
			\left(
			\begin{smallmatrix}
				1
				\\
				1
			\end{smallmatrix}
			\right)	\big),	\quad
		\Lambda
		:= \big(\tfrac 14(1, -1),0\big)
	$$
	and since $A_k \in \mathbb{R}^{1 \times N_{k-1}}$ and $b_k \in \mathbb{R}^1$ we define
	$$
		\tilde{\Phi}^1\! := \!\!
		\bigg(
			(A_1^1,b_1^1),
			\dots,
			(A_{k-2}^1,b_{k-2}^1),
			\left(\left(
			\begin{smallmatrix}
				A_{k-1}^1
				\\
				0_{\mathbb{R}^{N_{k-2}}}
			\end{smallmatrix}
			\right),
			\left(
			\begin{smallmatrix}
				b_{k-1}^1
				\\
				1
			\end{smallmatrix}
			\right)
			\right),
			\left(\left(
			\begin{smallmatrix}
				A_k^1, &1
				\\
				-A_k^1, &1
			\end{smallmatrix}
			\right),
			\left(
			\begin{smallmatrix}
				b_k^1
				\\
				-b_k^1
			\end{smallmatrix}
			\right)
			\right),
			\Gamma,
			\dots,
			\Gamma,
			\Lambda\!
		\bigg),
	$$
	where  $\Gamma$ appears $\ell - k - 1$ times,
	so that $L(\tilde{\Phi}^1) = \ell$.
	Where we use the fact that the realization $R_{\varrho_2}\Phi^1(x)$
	is bounded in $[-1,1]$ for any given date $x$.
	Hence $R_{\varrho_2}\Gamma : [-1,1]^2 \rightarrow [-1,1]^2$
	represents the identity in dimension $2$,
	that is
	$$
	R_{\varrho_2}\Gamma\left( \varrho_2\binom{R_{\varrho_2}\Phi^1(x)+1}
		{-R_{\varrho_2}\Phi^1(x)+1}\right)
	=
	\binom{R_{\varrho_2}\Phi^1(x)+1}{-R_{\varrho_2}\Phi^1(x)+1}.
	$$
	Hence the $j$-fold composition of $\Gamma$
	is defined as 
	$\Gamma^j = \Big( \underbrace{\Gamma,\dots, \Gamma}_{j \text{ times}}\Big)$ 
	hence the realization 
	$$
	R_{\varrho_2}\Gamma^j\left( \varrho_2\binom{R_{\varrho_2}\Phi^1(x)+1}{-R_{\varrho_2}\Phi^1(x)+1}\right)
	=
	\binom{R_{\varrho_2}\Phi^1(x)+1}{-R_{\varrho_2}\Phi^1(x)+1}.
	$$
	Then, we have
	$$
	R_{\varrho_2}\Lambda
	\left(
		\varrho_2\binom{R_{\varrho_2}\Phi^1(x)+1}{-R_{\varrho_2}\Phi^1(x)+1}
	\right)
	=
	R_{\varrho_2}\Phi^1(x).
	$$
	In conclusion,  we get
	$R_{\varrho_2} \tilde{\Phi}^1 = R_{\varrho_2} \Phi^1 $,
	such that	
	$\|{\tilde{\Phi}^1}\|_{\mathcal{NN}} \leq \boldsymbol{c}(n) $
	and
	$W({\tilde{\Phi}^1}) \leq 2 W(\Phi^1) + 6 (\ell - k-1)+2 \leq 8 n$.

	We write
	$\tilde{\Phi}^1 = \big( (\tilde{A}_1^1, \tilde{b}_1^1),
	\dots, (\tilde{A}_\ell^1, \tilde{A}_\ell^1) \big)$
	such that
	$\tilde{A}^1_j\in \mathbb{R}^{\tilde{N}_{j}\times \tilde{N}_{j-1}}$
	and $\tilde{b}^1_j\in \mathbb{R}^{\tilde{N}_{j}}$,
	$j\in \{1,\dots, \ell\}$
	where $\tilde{N}_0 =d$ and $\tilde{N}_\ell= 1$.
	We define 
	$$
	\Psi := \Big(
	\Theta_1,	\dots,	\Theta_{\ell-1},
	\big(
	(\tilde{A}^1_\ell \mid A^2_\ell),	\tilde{b}^1_\ell + b^2_\ell
	\big)
	\Big),
	$$
	where
	\[
	\Theta_1
	:= \left(
	\left(
	\begin{smallmatrix}
		\tilde{A}^1_1
		\\
		A_1^2 
	\end{smallmatrix}
	\right) ,
	\left(
	\begin{smallmatrix}
		\tilde{b}^1_1
		\\
		b_1^2
	\end{smallmatrix}
	\right)
	\right)
	\text{and} \quad
	\Theta_m
	:= \left(
	\left(
	\begin{smallmatrix}
		\tilde{A}^1_m & 0 
		\\
		0   & A_m^2
	\end{smallmatrix}
	\right) ,
	\left(
	\begin{smallmatrix}
		c_m \vphantom{\tilde{A}^1_m}
		\\
		e_m \vphantom{A_m^2}
	\end{smallmatrix}
	\right)
	\right)
	\text{for } m \in \{ 2,\dots,\ell-1 \}.
	\]
	We conclude that
	$R_{\varrho_2} \Psi =
	R_{\varrho_2}\tilde{\Phi}^1 + R_{\varrho_2} \Phi^2$.

	Moreover, we have
	\begin{align*}
		W(\Psi) &\leq W(\widetilde{\Phi}^1) + W(\Phi^2) \leq 9 n
		\\[1ex]
		\|{\Psi}\|_{\mathcal{NN}} &\leq
			\boldsymbol{c}(n) \leq \boldsymbol{c}(9n),
		\\[1ex]
		L(\Psi) &= \ell \leq \boldsymbol{\ell}(n) \leq \boldsymbol{\ell}(9n),
	\end{align*}
	Where we used the fact that $\boldsymbol{\ell}$
	and $\boldsymbol{c}$ are non-decreasing and that $\ell \leq n$.
	All in all, it follows that
	$\Sigma^{\boldsymbol{\ell}, \boldsymbol{c}, \varrho_2}_n
	+ \Sigma^{\boldsymbol{\ell}, \boldsymbol{c}, \varrho_2}_n
	\in \Sigma^{\boldsymbol{\ell}, \boldsymbol{c}, \varrho_2}_{9n}$,
	which conclude the lemma.

\end{proof}

\begin{proof}[Proof of \Cref{lem:ApproximationSpaceProperties}]
	The proof follows from Lemma \ref{lem:NetworkSetsClosedUnderSummation}
	and similar arguments as in the proof of \cite[Lemma 2.1]{grohs2021}.
	The details are left for interesting reader to check.
\end{proof}
\end{document}